\RequirePackage{fix-cm}
\documentclass[twocolumn]{svjour3}          % twocolumn
\smartqed  % flush right qed marks, e.g. at end of proof
\usepackage{graphicx}
\usepackage{mathptmx}      % use Times fonts if available on your TeX system
\usepackage{amsmath}
\usepackage{amssymb}

\usepackage{natbib}

\usepackage[colorlinks,linkcolor=red,anchorcolor=blue,citecolor=blue,CJKbookmarks=True]{hyperref}
\usepackage[american]{babel}
\usepackage{microtype}
\usepackage{multirow}
\usepackage{booktabs}
\usepackage{subfig}
\allowdisplaybreaks[3]
\begin{document}

\title{R2-D2: Repetitive Reprediction Deep Decipher for Semi-Supervised Deep Learning\thanks{
This work is supported by the National Natural Science Foundation of China (61772256, 61921006).}
}


\author{Guo-Hua Wang \and
	Jianxin Wu\thanks{J. Wu is the corresponding author.}
}


\institute{Guo-Hua Wang \at
              National Key Laboratory for Novel Software Technology, Nanjing University, Nanjing 210023, China. \\
              \email{wangguohua@lamda.nju.edu.cn}           %  \\
			\and
			Jianxin Wu \at
			National Key Laboratory for Novel Software Technology, Nanjing University, Nanjing 210023, China. \\
			\email{wujx2001@gmail.com}           %  \\
}

\date{Received: date / Accepted: date}

\maketitle

\begin{abstract}
Most recent semi-supervised deep learning (deep SSL) methods used a similar paradigm: use network predictions to update pseudo-labels and use pseudo-labels to update network parameters iteratively. However, they lack theoretical support and cannot explain why predictions are good candidates for pseudo-labels in the deep learning paradigm. In this paper, we propose a principled end-to-end framework named deep decipher (D2) for SSL. Within the D2 framework, we prove that pseudo-labels are related to network predictions by an exponential link function, which gives a theoretical support for using predictions as pseudo-labels. Furthermore, we demonstrate that updating pseudo-labels by network predictions will make them uncertain. To mitigate this problem, we propose a training strategy called repetitive reprediction (R2). Finally, the proposed R2-D2 method is tested on the large-scale ImageNet dataset and outperforms state-of-the-art methods by 5 percentage points.
\keywords{semi-supervised learning \and deep learning \and image classification}
\end{abstract}

\section{Introduction}
Deep learning has achieved state-of-the-art results on many visual recognition tasks. However, training these models often needs large-scale datasets such as ImageNet~\citep{imagenet}. Nowadays, it is easy to collect images by search engines, but image annotation is expensive and time-consuming. Semi-supervised learning (SSL) is a paradigm to learn a model with a few labeled data and massive amounts of unlabeled data. With the help of unlabeled data, the model performance may be improved.

With a supervised loss, unlabeled data can be used in training by assigning pseudo-labels to them.
Many state-of-the-art methods on semi-supervised deep learning used pseudo-labels implicitly. Temporal Ensembling~\citep{Temporal_Ensembling} used the moving average of network predictions as pseudo-labels. Mean Teacher~\citep{Mean_teacher} and Deep Co-training~\citep{DCT_2018_ECCV} employed another network to generate pseudo-labels. However, they produced or updated pseudo-labels in ad-hoc manners.
Although these methods worked well in practice, there are few theories to support them. A mystery in deep SSL arises: why can predictions work well as pseudo-labels?

In this paper, we propose an end-to-end framework called deep decipher (D2). Inspired by \citet{PENCIL}, we treat pseudo-labels as variables and update them by back-propagation, which are also learned from data. The D2 framework specifies a well-defined optimization problem, which can be properly interpreted as a maximum likelihood estimation over two set of variables (the network parameters and the pseudo-labels). Within deep decipher, we prove that there exists an exponential relationship between pseudo-labels and network predictions, leading to a theoretical support for using network predictions as pseudo-labels. Then, we further analyze the D2 framework and prove that pseudo-labels will become flat (i.e., their entropy is high) during training and there is an equality constraint bias in it. To mitigate these problems, we propose a simple but effective strategy, repetitive reprediction (R2). The improved D2 framework is named R2-D2 and obtains state-of-the-art results on several SSL problems.

Our contributions are as follows. 

\begin{itemize}
	\item We propose D2, a deep learning framework that deciphers the relationship between predictions and pseudo-labels. D2 updates pseudo-labels by back-propagation. To the best of our knowledge, D2 is the first deep SSL method that learns pseudo-labels from data end-to-end.
	\item Within D2, we prove that pseudo-labels are exponentially transformed from the predictions. Hence, it is reasonable for previous works to use network predictions as pseudo-labels. Meanwhile, many SSL methods can be considered as special cases of D2 in terms of certain aspects.
	\item To further boost D2's performance, we find some shortcomings of D2. In particular, we prove that pseudo-labels will become flat during the optimization. To mitigate this problem, we propose a simple but effective remedy, R2. We tested the R2-D2 method on ImageNet and it outperforms state-of-the-arts by a large margin. On small-scale datasets like CIFAR-10~\citep{cifar}, R2-D2 also produces state-of-the-art results.
\end{itemize}

\section{Related Works}

We first briefly review deep SSL methods and the related works that inspired this paper.

\citet{pseudo_label} is an early work on training deep SSL models by pseudo-labels, which picks the class with the maximum predicted probability as pseudo-label for unlabeled images and tested only on a samll-scale dataset MNIST~\citep{LeNet}. Label propagation~\citep{label_prop} can be seen as a form of pseudo-labels. Based on some metric, label propagation pushes the label information of each sample to the near samples. \citet{deep_label_prop} apply label propagation to deep learning models. \citet{IJCV_semi} use label propagation to solve the exhaustively propagating pairwise constraints problem. \citet{cvpr2019_pseudo_label} use the manifold assumption to generate pseudo-labels for unlabeled data. However, their method is complicated and relies on other SSL methods to produce state-of-the-art results.

Several recent state-of-the-art deep SSL methods can be considered as using pseudo-labels implicitly. Temporal ensembling~\citep{Temporal_Ensembling} proposes making the current prediction and the pseudo-labels consistent, where the pseudo-labels take into account the network predictions over multiple previous training epochs. Extending this idea, Mean Teacher~\citep{Mean_teacher} employs a secondary model, which uses the exponential moving average weights to generate pseudo-labels. Virtual Adversarial Training~\citep{VAT} uses network predictions as pseudo-labels, then they want the network predictions under adversarial perturbation to be consistent with pseudo-labels. Deep Co-Training~\citep{DCT_2018_ECCV} employs many networks and uses one network to generate pseudo-labels for training other networks. 

We notice that they all use the network predictions as pseudo-labels but a theory explaining its rationale in the deep learning context is missing. Within our D2 framework, we demonstrate that pseudo-labels will indeed be related to network predictions. That gives a support to using network predictions as pseudo-labels. Moreover, pseudo-labels of previous works were designed manually and ad-hoc, but our pseudo-labels are updated by training the end-to-end framework. Many previous SSL methods can also be considered as special cases of the D2 framework in terms of certain aspects in these methods.

There are some previous works in other fields that inspired this work.
Deep label distribution learning~\citep{DLDL} inspires us to use label distributions to encode the pseudo-labels. 
\citet{Noisy_Labels} studies the label noise problem. They find it is possible to update noisy labels to make them more precise during the training. PENCIL~\citep{PENCIL} proposes an end-to-end framework to train the network and optimize the noisy labels together. Our method is inspired by PENCIL~\citep{PENCIL}. In addition, inspired by \citet{TCP}, we analyze our algorithm from the gradient perspective. \footnote{Preliminary studies of the proposed R2-D2 method appeared as a conference presentation \citep{R2-D2}, available at \url{https://arxiv.org/abs/1908.04345}. }

\section{The R2-D2 Method}

We define the notations first. Column vectors and matrices are denoted in bold (e.g., $\mathbf{x}, \mathbf{X}$). When $\mathbf{x}\in \mathbb{R}^d$,  $x_i$ is the $i$-th element of the vector $\mathbf{x}$, $i\in[d]$, where $[d]:=\{1,2,\dots,d\}$. $\mathbf{w}_i$ denote the $i$-th column of matrix $\mathbf{W}\in\mathbb{R}^{d\times l}$, $i\in[l]$. And, we assume the dataset has $N$ classes.

\subsection{Deep decipher (D2)}

\begin{figure*}[t]
	\centering
	\includegraphics[width=0.65\linewidth]{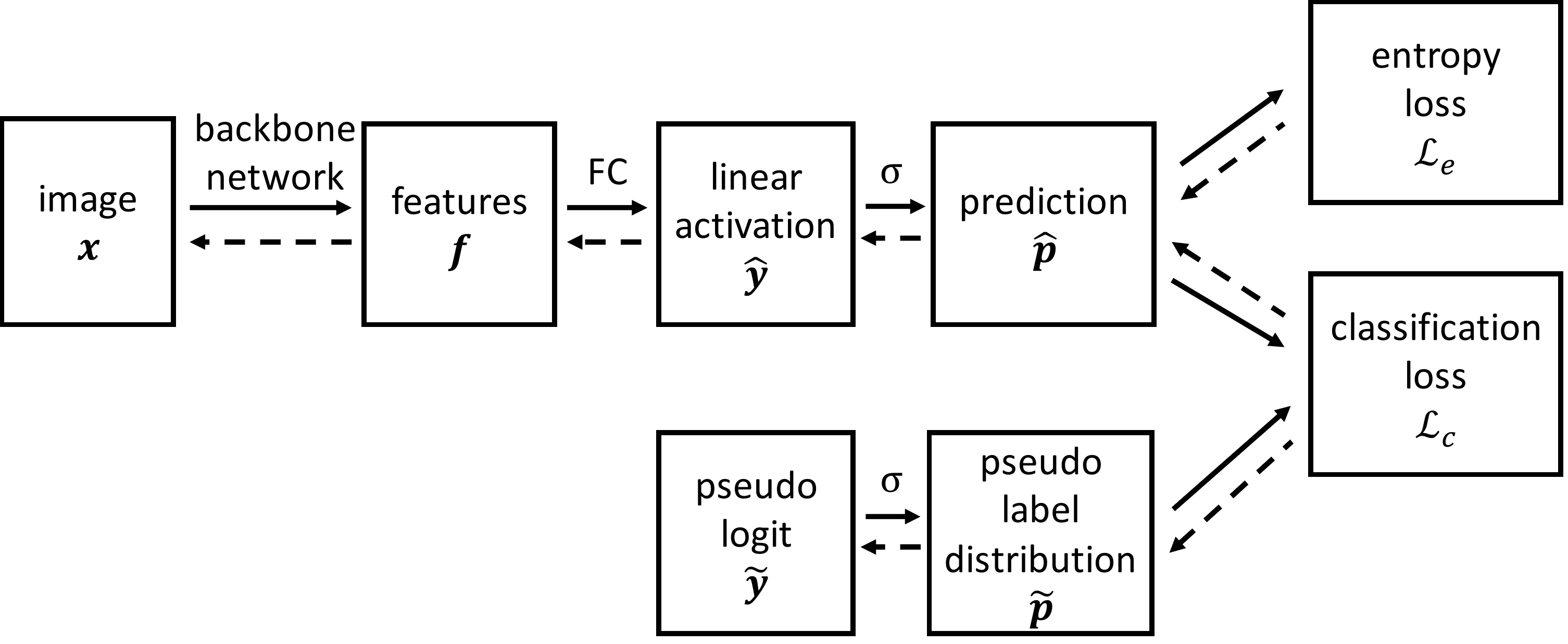}
	\caption{The pipeline of D2. Solid lines and dashed lines represent the forward and back-propagation processes, respectively.}
	\label{fig:framework}
\end{figure*}

Figure~\ref{fig:framework} shows the D2 pipeline, which is inspired by \citet{PENCIL}. Given an input image $\mathbf{x}$, D2 can employ any backbone network to generate feature $\mathbf{f}\in\mathbb{R}^D$. Then, the linear activation $\mathbf{\hat{y}}\in \mathbb{R}^N$ is computed as $\mathbf{\hat{y}}=\mathbf{W}^\mathsf{T}\mathbf{f}$, where $\mathbf{W}\in\mathbb{R}^{D\times N}$ are weights of the FC layer and we omit the bias term for simplicity. The softmax function is denoted as $\sigma(\mathbf{y}):\mathbb{R}^N\rightarrow\mathbb{R}^N$ and $\sigma(\mathbf{y})_i=\frac{\exp\left(y_i\right)}{\sum_{j=1}^{N}\exp\left(y_j\right)}$. Then, the prediction $\hat{\mathbf{p}}$ is calculated as $\hat{\mathbf{p}}=\sigma(\mathbf{\hat{y}})$
, hence
\begin{equation}
\hat{p}_n
=\sigma(\mathbf{\hat{y}})_n
=\sigma(\mathbf{W}^\mathsf{T}\mathbf{f})_n
=\frac{\exp(\mathbf{w}_n^\mathsf{T}\mathbf{f})}{\sum_{i=1}^{N}\exp(\mathbf{w}_i^\mathsf{T}\mathbf{f})}\,.
\end{equation}

We define $\tilde{\mathbf{y}}$ as the pseudo logit which is an unconstrained variable and \emph{can} be updated by back-propagation. Then, the pseudo label is calculated as $\tilde{\mathbf{p}} = \sigma(\tilde{\mathbf{y}})$ and it is a valid probability distribution. 

In the training, the D2 framework is initialized as follows. Firstly, we train the backbone network using only labeled examples, and use this trained network as the backbone network and FC in Figure~\ref{fig:framework}. For labeled examples, $\tilde{\mathbf{y}}$ is initialized by $K\mathbf{y}$, in which $K=10$ and $\mathbf{y}$ is the groundtruth label in the one-hot encoding. Note that $\tilde{\mathbf{y}}$ of labeled examples will \emph{not} be updated during D2 training. For unlabeled examples, we use the trained network to predict $\tilde{\mathbf{y}}$. That means we use the FC layer activation $\hat{\mathbf{y}}$ as the initial value of $\tilde{\mathbf{y}}$. The process of initializing pseudo-labels is called predicting pseudo-labels in this paper.
In the testing, we use the backbone network with FC layer to make predictions and the branch of pseudo-labels is removed.

Our loss function consists of $\mathcal{L}_c$ and $\mathcal{L}_e$. $\mathcal{L}_c$ is the classification loss and defined as $KL(\hat{\mathbf{p}}||\tilde{\mathbf{p}})$ as in \citet{PENCIL}, which is different from the classic KL-loss $ KL(\tilde{\mathbf{p}}||\hat{\mathbf{p}})$. $\mathcal{L}_c$ is used to make the network predictions match the pseudo-labels. $\mathcal{L}_e$ is the entropy loss, defined as $-\sum_{j=1}^{N}\hat{p}_j\log(\hat{p}_j)$. 
Minimizing the entropy of the network prediction can encourage the network to peak at only one category. 
So our loss function is defined as
\begin{align}
\label{loss-function}
\mathcal{L}
&=\alpha\mathcal{L}_c + \beta\mathcal{L}_e \notag \\
&=\alpha\sum_{j=1}^{N}\hat{p}_j\left[\log(\hat{p}_j)-\log(\tilde{p}_j)\right]-\beta\sum_{j=1}^{N}\hat{p}_j\log(\hat{p}_j)\,,
\end{align}
where $\alpha$ and $\beta$ are two hyperparameters. Although there are two hyperparameters in D2, we always set $\alpha=0.1$ and $\beta=0.03$ in all our experiments.

Then, we show that we can decipher the relationship between pseudo-labels and network predictions in D2, as shown by Theorem~\ref{thm:1}.
\begin{theorem}
	\label{thm:1}
	Suppose D2 is trained by SGD with the loss function $\mathcal{L}=\alpha\mathcal{L}_c + \beta\mathcal{L}_e$. Let $\hat{\mathbf{p}}$ denote the prediction by the network for one example and $\hat{p}_n$ is the largest value in $\hat{\mathbf{p}}$. After the optimization algorithm converges, we have $\tilde{p}_n\rightarrow\exp(-\frac{\mathcal{L}}{\alpha})\left(\hat{p}_n\right)^{1-\frac{\beta}{\alpha}}$.
\end{theorem}

\begin{proof}
	First, the loss function can be rewritten as 
	\begin{align}
	\mathcal{L}
	&=(\alpha-\beta)\sum_{j=1}^{N}\sigma\left(\mathbf{W}^\mathsf{T}\mathbf{f}\right)_j\log\left(\sigma\left(\mathbf{W}^\mathsf{T}\mathbf{f}\right)_j\right) \notag \\
	&\quad
	-\alpha\sum_{j=1}^{N}\sigma\left(\mathbf{W}^\mathsf{T}\mathbf{f}\right)_j\log(\tilde{p}_j)\,.
	\end{align}
	It is easy to see
	\begin{align}
	\frac{\partial\sigma\left(\mathbf{W}^\mathsf{T}\mathbf{f}\right)_j}{\partial\mathbf{w}_n}
	&=\mathbb{I}(j=n)\sigma\left(\mathbf{W}^\mathsf{T}\mathbf{f}\right)_j\mathbf{f} \notag \\
	&\quad
	-\sigma\left(\mathbf{W}^\mathsf{T}\mathbf{f}\right)_j\sigma\left(\mathbf{W}^\mathsf{T}\mathbf{f}\right)_n\mathbf{f}\,,
	\end{align}
	in where $\mathbb{I}(\cdot)$ is the indicator function.
	Now we can compute the gradient of $\mathcal{L}$ with respect to $\mathbf{w}_n$:
	\begin{align}
	\frac{\partial\mathcal{L}}{\partial\mathbf{w}_n}
	&=(\alpha-\beta)
	\sum_{j=1}^{N}
	\left[
	\frac{\partial\sigma\left(\hat{\mathbf{y}}\right)_j}{\partial\mathbf{w}_n}\log\left(\sigma\left(\hat{\mathbf{y}}\right)_j\right) \right. \notag \\
	&\quad\left.
	+\sigma\left(\hat{\mathbf{y}}\right)_j\frac{\partial\log\left(\sigma\left(\hat{\mathbf{y}}\right)_j\right)}{\partial\mathbf{w}_n} \right]
	\notag \\
	&\quad
	-\alpha\sum_{j=1}^{N}\frac{\partial\sigma\left(\hat{\mathbf{y}}\right)_j}{\partial\mathbf{w}_n}\log(\tilde{p}_j)\\
	&=
	(\alpha-\beta)\sum_{j=1}^{N}
	\left[
	\mathbb{I}(j=n)
	-\sigma\left(\hat{\mathbf{y}}\right)_n
	\right]
	\sigma\left(\hat{\mathbf{y}}\right)_j\mathbf{f}
	\log\left(\sigma\left(\hat{\mathbf{y}}\right)_j\right) \notag\\
	&\quad 
	+(\alpha-\beta)\sum_{j=1}^{N}
	\sigma\left(\hat{\mathbf{y}}\right)_j
	\left(
	\mathbb{I}(j=n)\mathbf{f}
	-\sigma\left(\hat{\mathbf{y}}\right)_n\mathbf{f}
	\right)\notag\\
	&\quad
	-\alpha\sum_{j=1}^{N}
	\left[
	\mathbb{I}(j=n)\sigma\left(\hat{\mathbf{y}}\right)_j\mathbf{f}
	-\sigma\left(\hat{\mathbf{y}}\right)_j\sigma\left(\hat{\mathbf{y}}\right)_n\mathbf{f}
	\right]
	\log(\tilde{p}_j)\\
	&=(\alpha-\beta)\sigma\left(\hat{\mathbf{y}}\right)_n\log\left(\sigma\left(\hat{\mathbf{y}}\right)_n\right)\mathbf{f}\notag\\
	&\quad 
	-(\alpha-\beta)\sum_{j=1}^{N}
	\sigma\left(\hat{\mathbf{y}}\right)_j\log\left(\sigma\left(\hat{\mathbf{y}}\right)_j\right)\sigma\left(\hat{\mathbf{y}}\right)_n\mathbf{f}\notag\\
	&\quad 
	+(\alpha-\beta)\sigma\left(\hat{\mathbf{y}}\right)_n\mathbf{f}
	-(\alpha-\beta)\sigma\left(\hat{\mathbf{y}}\right)_n\mathbf{f}\sum_{j=1}^{N}\sigma\left(\hat{\mathbf{y}}\right)_j\notag\\
	&\quad 
	-\alpha\sigma\left(\hat{\mathbf{y}}\right)_n\log(\tilde{p}_n)\mathbf{f}
	+\alpha\sum_{j=1}^{N}\sigma\left(\hat{\mathbf{y}}\right)_j\log(\tilde{p}_j)\sigma\left(\hat{\mathbf{y}}\right)_n\mathbf{f}\\
	&=(\alpha-\beta)\sigma\left(\hat{\mathbf{y}}\right)_n\log\left(\sigma\left(\hat{\mathbf{y}}\right)_n\right)\mathbf{f}
	-\alpha\sigma\left(\hat{\mathbf{y}}\right)_n\log(\tilde{p}_n)\mathbf{f}\notag\\
	&\quad
	-\sigma\left(\hat{\mathbf{y}}\right)_n\mathbf{f}
	\left[
	(\alpha-\beta)\sum_{j=1}^{N}
	\sigma\left(\hat{\mathbf{y}}\right)_j\log\left(\sigma\left(\hat{\mathbf{y}}\right)_j\right)
	\right. \notag \\
	&\quad \left.
	-\alpha\sum_{j=1}^{N}\sigma\left(\hat{\mathbf{y}}\right)_j\log(\tilde{p}_j)
	\right]\notag\\
	&\quad
	+(\alpha-\beta)\sigma\left(\hat{\mathbf{y}}\right)_n\mathbf{f}
	-(\alpha-\beta)\sigma\left(\hat{\mathbf{y}}\right)_n\mathbf{f}\\
	&=
	\left[(\alpha-\beta)\log\left(\sigma\left(\hat{\mathbf{y}}\right)_n\right)
	-\alpha\log(\tilde{p}_n)
	\right]\sigma\left(\hat{\mathbf{y}}\right)_n\mathbf{f} \notag \\
	&\quad
	-\mathcal{L}\sigma\left(\hat{\mathbf{y}}\right)_n\mathbf{f}\\
	&=
	\left[(\alpha-\beta)\log\left(\sigma\left(\hat{\mathbf{y}}\right)_n\right)
	-\alpha\log(\tilde{p}_n)
	-\mathcal{L}
	\right]\sigma\left(\hat{\mathbf{y}}\right)_n\mathbf{f}\\
	&=
	\left[(\alpha-\beta)\log\left(\hat{p}_n\right)
	-\alpha\log(\tilde{p}_n)
	-\mathcal{L}
	\right]\hat{p}_n\mathbf{f}\,.
	\end{align}
	
	During training, we expect the optimization algorithm can converge and finally $\frac{\partial\mathcal{L}}{\partial\mathbf{w}_n}\rightarrow \mathbf{0}$. Because $\mathbf{f}$ will not be $\mathbf{0}$, we conclude that $\left[(\alpha-\beta)\log\left(\hat{p}_n\right)
	-\alpha\log(\tilde{p}_n)
	-\mathcal{L}
	\right]\hat{p}_n\rightarrow 0$. Because $\sum_{i=1}^{N}\hat{p}_i=1$, consider the fact that $\hat{p}_n$ is the largest value in $\{\hat{p}_1, \hat{p}_1, \dots, \hat{p}_N\}$, then $\hat{p}_n\not\rightarrow 0$ at the end of training. So we have $\left[(\alpha-\beta)\log\left(\hat{p}_n\right)
	-\alpha\log(\tilde{p}_n)
	-\mathcal{L}
	\right]\rightarrow 0$, which easily translates to  $\tilde{p}_n\rightarrow\exp(-\frac{\mathcal{L}}{\alpha})\left(\hat{p}_n\right)^{1-\frac{\beta}{\alpha}}$.
	\qed
\end{proof}

\begin{figure}
	\centering	
	\includegraphics[width=0.7\linewidth]{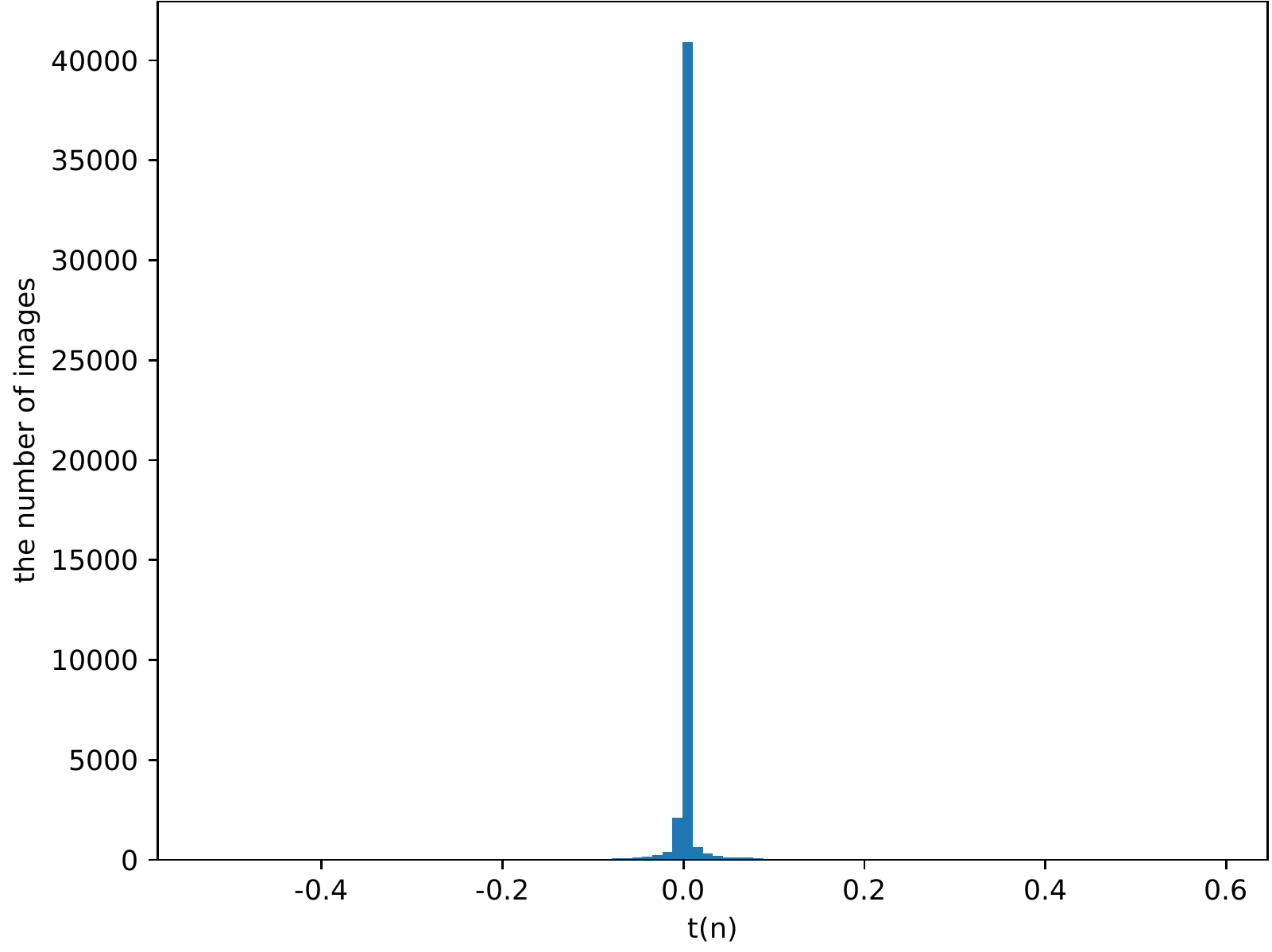}
	\caption{Distribution of $t(n)$ on the whole CIFAR-10 dataset at the end of the D2 training. $t(n)$ is defined as $(\alpha-\beta)\log\left(\hat{p}_n\right)-\alpha\log(\tilde{p}_n)-\mathcal{L}$. We can see $t(n)=0$ for almost all images, where $n$ is calculated according to each image.}
	\label{fig:t(n)}
\end{figure}

We would like to show experimental results for verifying Theorem~\ref{thm:1}. Let $t(n)$ denote $(\alpha-\beta)\log\left(\hat{p}_n\right)-\alpha\log(\tilde{p}_n)-\mathcal{L}$. Now, consider a single sample, suppose $\hat{\mathbf{p}}$ will get the largest value at $n$ where $n\in\{1,2,\dots,N\}$. Then it is expected that $\hat{p}_n\rightarrow 1$ and $t(n)\rightarrow 0$ at the end of training. Figure~\ref{fig:t(n)} shows the distribution of $t(n)$ on the whole CIFAR-10 dataset, where $n$ is calculated according to different samples. The distribution is almost gathered around $0$. 
So we also observed empirically that $\tilde{p}_n\rightarrow\exp(-\frac{\mathcal{L}}{\alpha})\left(\hat{p}_n\right)^{1-\frac{\beta}{\alpha}}$, where $n$ is the class predicted by the network.

Theorem~\ref{thm:1} tells us
$\tilde{p}_n$ converges to $\exp(-\frac{\mathcal{L}}{\alpha})\left(\hat{p}_n\right)^{1-\frac{\beta}{\alpha}}$ during the optimization. And at last, we expect that $\tilde{p}_n=\exp(-\frac{\mathcal{L}}{\alpha})\left(\hat{p}_n\right)^{1-\frac{\beta}{\alpha}}$, in which $n$ is the class predicted by the network. 
%Appendix~\ref{apx:exp:thm:1} shows empirical validation results for this theorem and it holds very well in practice.
In other words, we have deciphered that there is an exponential link between pseudo-labels and predictions. From $\tilde{p}_n\rightarrow\exp(-\frac{\mathcal{L}}{\alpha})\left(\hat{p}_n\right)^{1-\frac{\beta}{\alpha}}$, we notice that $\tilde{p}_n$ is approximately proportional to $\hat{p}_n^{1-\frac{\beta}{\alpha}}$. That gives a theoretical support to use network predictions as pseudo-labels. And, it is required that $1-\frac{\beta}{\alpha}>0$ to make pseudo-labels and network predictions consistent. We must set $\alpha>\beta$. In our experiments, if we set $\alpha < \beta$, the training will indeed fail miserably.

Next, we analyze how $\tilde{\mathbf{y}}$ is updated in D2. With the loss function $\mathcal{L}$, the gradients of $\mathcal{L}$ with respect to $\tilde{y}_n$ is
\begin{align}
\frac{\partial\mathcal{L}}{\partial\tilde{y}_n}
&=\sum_{k=1}^{N}
\frac{\partial\mathcal{L}}{\partial\tilde{p}_k}
\frac{\partial\tilde{p}_k}{\partial\tilde{y}_n}
\\
&=-\alpha\sum_{k=1}^{N}
\frac{\sigma\left(\hat{\mathbf{y}}\right)_k}{\tilde{p}_k}
\left(
\mathbb{I}(k=n)\sigma\left(\tilde{\mathbf{y}}\right)_k
-\sigma\left(\tilde{\mathbf{y}}\right)_k\sigma\left(\tilde{\mathbf{y}}\right)_n
\right)
\\
&=-\alpha\sum_{k=1}^{N}
\frac{\sigma\left(\hat{\mathbf{y}}\right)_k}{\sigma\left(\tilde{\mathbf{y}}\right)_k}
\left(
\mathbb{I}(k=n)\sigma\left(\tilde{\mathbf{y}}\right)_k
-\sigma\left(\tilde{\mathbf{y}}\right)_k\sigma\left(\tilde{\mathbf{y}}\right)_n
\right)
\\
&=-\alpha\sum_{k=1}^{N}
\sigma\left(\hat{\mathbf{y}}\right)_k
\left(
\mathbb{I}(k=n)
-\sigma\left(\tilde{\mathbf{y}}\right)_n
\right)
\\
&=-\alpha\sum_{k=1}^{N}
\sigma\left(\hat{\mathbf{y}}\right)_k
\mathbb{I}(k=n)
+\alpha\sum_{k=1}^{N}
\sigma\left(\hat{\mathbf{y}}\right)_k
\sigma\left(\tilde{\mathbf{y}}\right)_n
\\
&=-\alpha
\sigma\left(\hat{\mathbf{y}}\right)_n
+\alpha\sigma\left(\tilde{\mathbf{y}}\right)_n\sum_{k=1}^{N}
\sigma\left(\hat{\mathbf{y}}\right)_k
\\
&=-\alpha
\sigma\left(\hat{\mathbf{y}}\right)_n
+\alpha\sigma\left(\tilde{\mathbf{y}}\right)_n
\,.
\end{align}

\begin{figure}
	\centering 
	\subfloat[]{\label{fig:sub1}\includegraphics[width=0.45\linewidth]{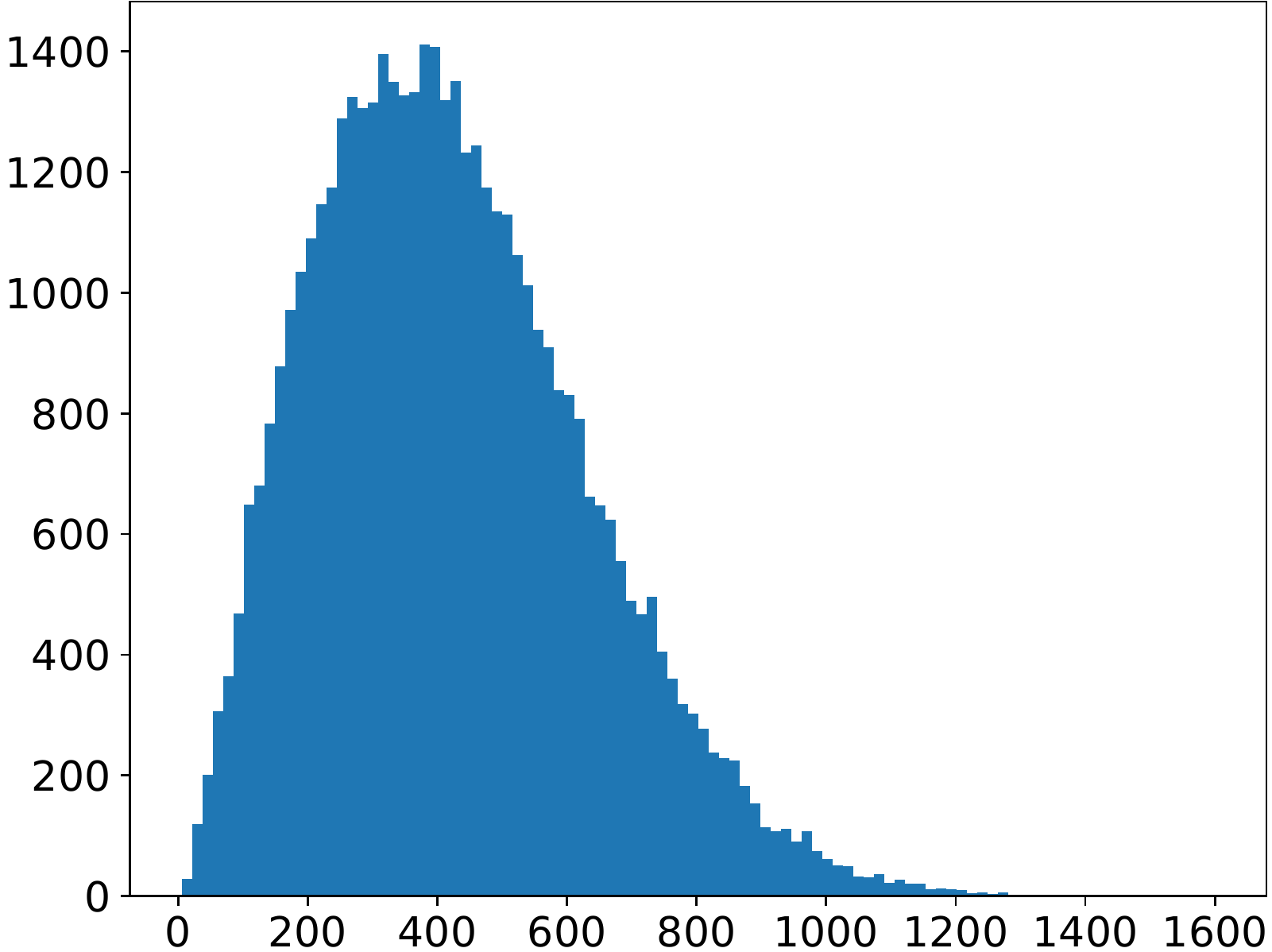}}\quad 
	\subfloat[]{\label{fig:sub2}\includegraphics[width=0.45\linewidth]{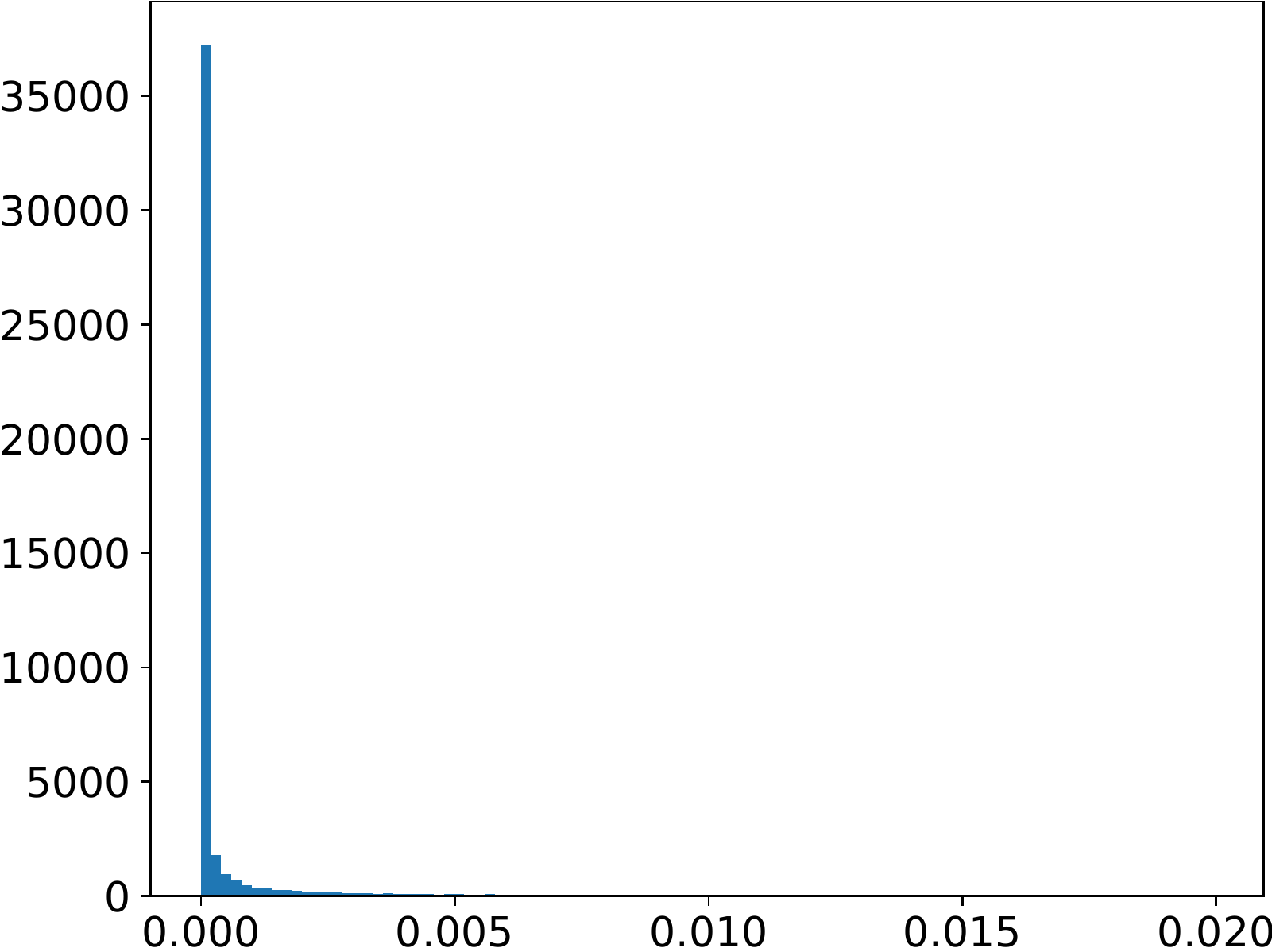}}\\ 
	\caption{\protect\subref{fig:sub1} shows the distribution of the number of images versus $\|\tilde{\mathbf{y}}\|_2$ in CIFAR-10. \protect\subref{fig:sub2} shows the distribution of the number of images versus $\|\frac{\partial\mathcal{L}}{\partial\tilde{\mathbf{y}}}\|_2$. Note that the ranges of x-axis are \emph{different} between \protect\subref{fig:sub1} and \protect\subref{fig:sub2}. From the figure, we can see the magnitude of  $\frac{\partial\mathcal{L}}{\partial\tilde{\mathbf{y}}}=-\alpha\sigma\left(\hat{\mathbf{y}}\right)+\alpha\sigma\left(\tilde{\mathbf{y}}\right)$ is far less than that of $\tilde{\mathbf{y}}$. So we use one more hyperparameter $\lambda$ rather than the overall learning rate to update the pseudo logit $\tilde{\mathbf{y}}$.}
	\label{fig:stage2-update} 
\end{figure}

By gradient descent, the pseudo logit $\tilde{\mathbf{y}}$ is updated by
\begin{equation}
\label{updating-formula}
\tilde{\mathbf{y}}\leftarrow \tilde{\mathbf{y}}-\lambda\frac{\partial\mathcal{L}}{\partial\tilde{\mathbf{y}}}
=\tilde{\mathbf{y}}
-\lambda\alpha\sigma\left(\tilde{\mathbf{y}}\right)
+\lambda\alpha\sigma\left(\hat{\mathbf{y}}\right)\,,
\end{equation}
where $\lambda$ is the learning rate for updating $\tilde{\mathbf{y}}$.
The reason we use one more hyperparameter $\lambda$ rather than the overall learning rate is because the magnitude of  $\frac{\partial\mathcal{L}}{\partial\tilde{\mathbf{y}}}=-\alpha\sigma\left(\hat{\mathbf{y}}\right)+\alpha\sigma\left(\tilde{\mathbf{y}}\right)$ is much smaller than that of $\tilde{\mathbf{y}}$ (in part due to the sigmoid transform) and the overall learning rate is too small to update the pseudo logit (cf. Figure~\ref{fig:stage2-update}).
We set $\lambda=4000$ in all our experiments.

The updating formulas in many previous works can be considered as special cases of that of D2. 
In Temporal Ensembling~\citep{Temporal_Ensembling}, the pseudo-labels $\tilde{\mathbf{p}}$ is a moving average of the network predictions $\hat{\mathbf{p}}$ during training. The updating formula is $\mathbf{P}\leftarrow\alpha\mathbf{P}+(1-\alpha)\hat{\mathbf{p}}$. To correct for the startup bias, the $\tilde{\mathbf{p}}$ needs to be divided by the factor $(1-\alpha^t)$, where $t$ is the number of epochs. So the updating formula of $\tilde{\mathbf{p}}$ is $\tilde{\mathbf{p}}\leftarrow \mathbf{P} / (1-\alpha^t)$. In Mean Teacher~\citep{Mean_teacher}, the $\tilde{\mathbf{p}}$ is the prediction of a teacher model which uses the exponential moving average weights of the student model. \citet{Noisy_Labels} proposed using the running average of the network predictions to estimate the groundtruth of the noisy label. However, their updating formula were designed manually and ad-hoc. In contrast, we treat pseudo-labels as updatable variables like the network parameters. These variables are learned by minimizing a well-defined loss function (cf. equation \ref{loss-function}). From a probabilistic perspective, it is well known that minimizing the KL loss is equivalent to maximum likelihood estimation, in which the backbone network's architecture defines the estimation's functional space while SGD optimizes over these variables (both the network parameters and the pseudo-labels). We do not need to manually specify how the pseudo-labels are generated. This process is natural and principled.

\subsection{An illustrative example}
\begin{figure*}[t] 
	\centering 
	\subfloat[]{\label{fig:mnist:sub1}\includegraphics[width=0.25\linewidth]{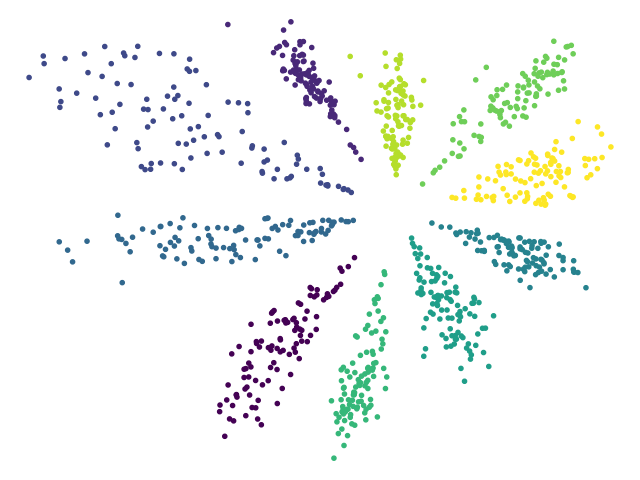}}
	\subfloat[]{\label{fig:mnist:sub2}\includegraphics[width=0.25\linewidth]{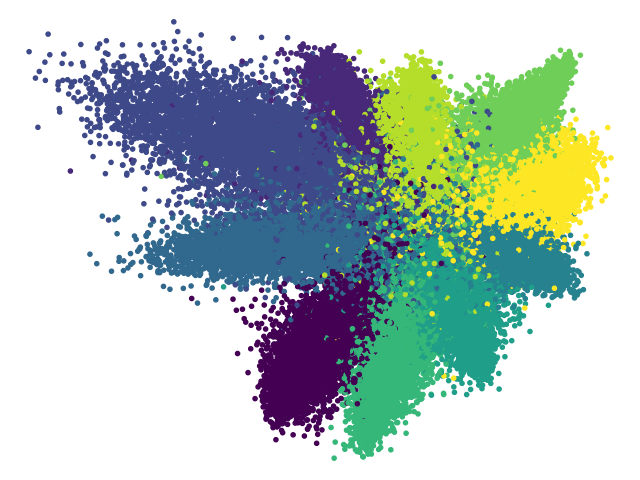}} 
	\subfloat[]{\label{fig:mnist:sub3}\includegraphics[width=0.25\linewidth]{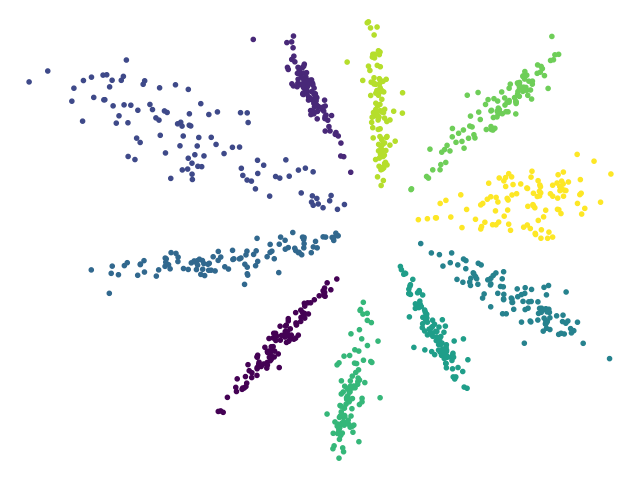}}
	\subfloat[]{\label{fig:mnist:sub4}\includegraphics[width=0.25\linewidth]{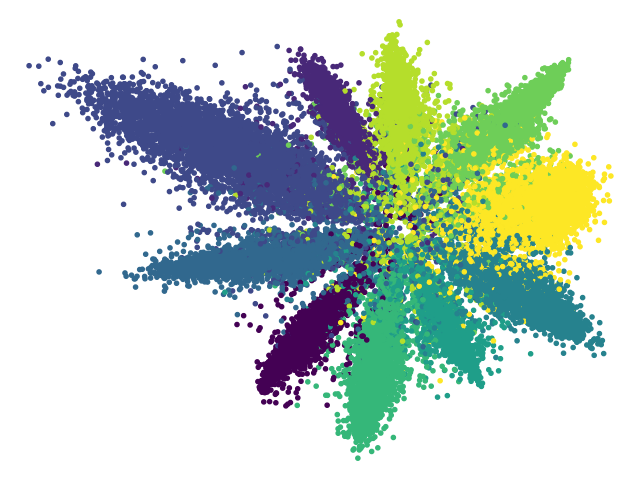}}
	\caption{Feature distribution on MNIST. First, LeNet was trained by labeled data. \protect\subref{fig:mnist:sub1} shows the the feature distribution of labeled images. Points with the same color belong to the same class. \protect\subref{fig:mnist:sub2} shows the feature distribution of both labeled and unlabeled images. Then, we used LeNet as the backbone network and trained the D2 framework. After training, \protect\subref{fig:mnist:sub3} and \protect\subref{fig:mnist:sub4} show the feature distribution of labeled images and all images, respectively. This figure needs to be viewed in color.}
	\label{fig:mnist}
\end{figure*}

Now, we use a toy example to explain how the D2 framework works intuitively. Inspired by \citet{TCP}, we use the LeNet~\citep{LeNet} as backbone structure and add two FC layers, in which the first FC layer learns a 2-D feature and the second FC layer projects the feature onto the class space. The network was trained on MNIST. Note that MNIST has 50000 images for training. We only used 1000 images as labeled images to train the network. Figure~\ref{fig:mnist:sub1} depicts the 2-D feature distribution of these 1000 images. We observe that features belonging to the same class will cluster together. Figure~\ref{fig:mnist:sub2} shows the feature distribution of both these 1000 labeled and other 49000 unlabeled images. Although the network did not train on the unlabeled images, features belonging to the same class are still roughly clustered.

Pseudo-labels in our D2 framework are probability distributions and initialized by network predictions. As Figure~\ref{fig:mnist:sub2} shows, features near the cluster center will have confident pseudo-labels and can be learned safely. However, features at the boundaries between clusters will have a pseudo-label whose corresponding distribution among different classes is flat rather than sharp. By training D2, the network will learn confident pseudo-labels first. Then it is expected that uncertain pseudo-labels will become more and more precise and confident by optimization. At last, each cluster will become more compact and the boundaries between different classes' features will become clear.  Figure~\ref{fig:mnist:sub4} depicts the feature distribution of all images after D2 training. Because the same class features of unlabeled images get closer, the same class features of labeled images will also get closer (cf. Figure~\ref{fig:mnist:sub3}). That is how unlabeled images help the training in our D2 framework.

\subsection{Repetitive reprediction (R2)}

Although D2 has worked well in practice (cf. Table~\ref{tab:ablation-result-table} column a), there are still some shortcomings in it. We will discuss two major ones. To mitigate these problems and further boost the performance, we propose a simple but effective strategy, repetitive reprediction (R2), to improve the D2 framework.

\begin{figure}
	\centering	
	\includegraphics[width=0.9\linewidth]{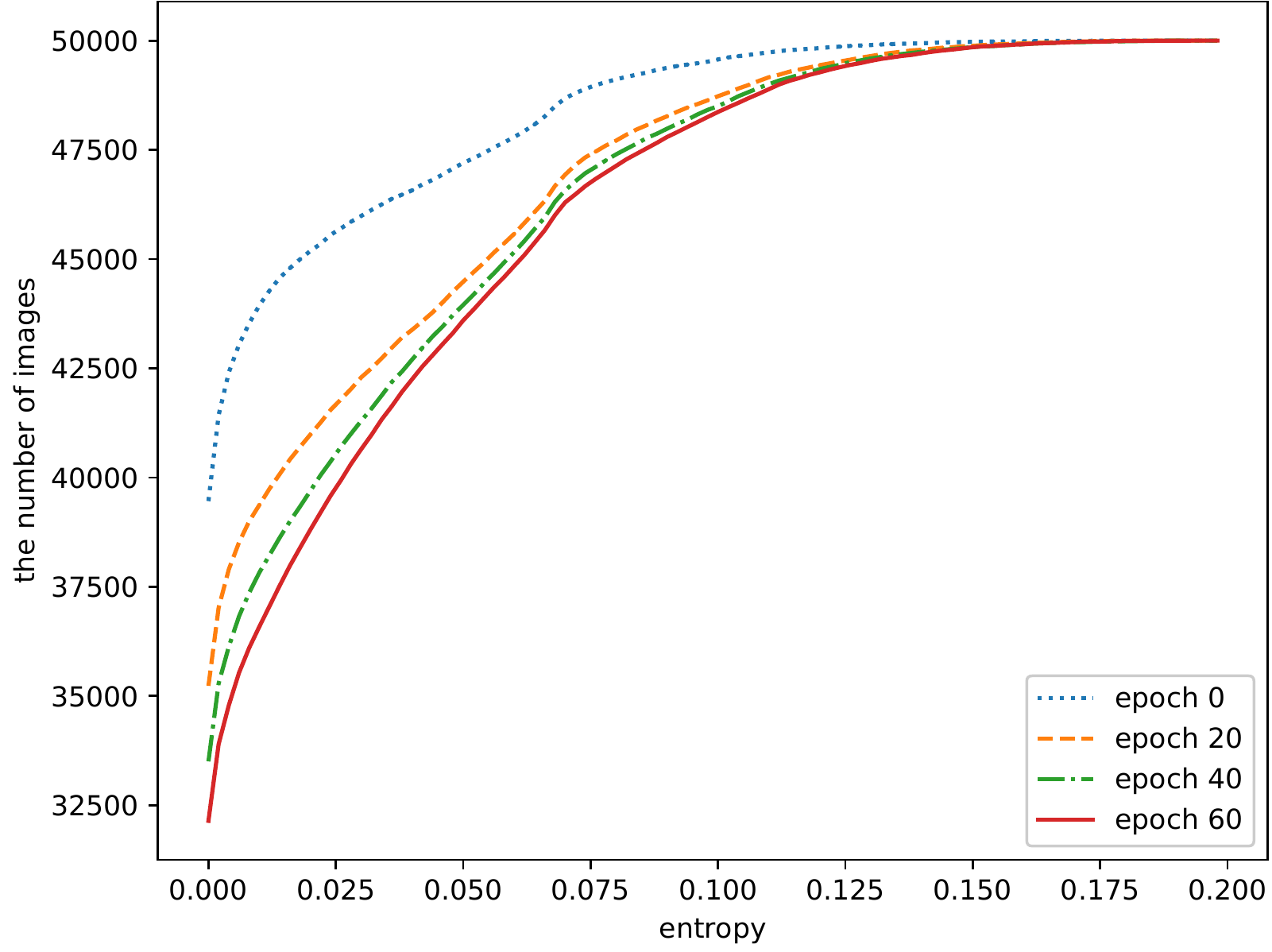}
	\caption{Cumulative distribution of the number of pseudo-labels versus the entropy. For each point $(n, e)$ on the line, it means there are $n$ images whose label entropies are less than $e$. Pseudo-labels will become flat as the D2 framework is trained more epochs. This figure is best viewed in color. }
	\label{fig:stage2-num_entropy}
\end{figure}

First, we expect pseudo-labels can become more confident along with D2's learning process. Unfortunately, we observed that more and more pseudo-labels become flat during training (cf. Figure~\ref{fig:stage2-num_entropy}). Below, we prove Theorem~\ref{thm:2} to explain why this adverse effect happens. %The experiments result can be found in Appendix~\ref{apx:exp:thm:2}.
\begin{theorem}
	\label{thm:2}
	Suppose D2 is trained by SGD with the loss function $\mathcal{L}=\alpha\mathcal{L}_c + \beta\mathcal{L}_e$. If $\tilde{p}_n=\exp(-\frac{\mathcal{L}}{\alpha})\left(\hat{p}_n\right)^{1-\frac{\beta}{\alpha}}$, we must have $\tilde{p}_n\leq \hat{p}_n$. 
\end{theorem}

\begin{proof}
	First, according to the loss function we defined, we have 
	\begin{align}
	\mathcal{L}
	&=\alpha\sum_{j=1}^{N}\hat{p}_j\left[\log(\hat{p}_j)-\log(\tilde{p}_j)\right]-\beta\sum_{j=1}^{N}\hat{p}_j\log(\hat{p}_j)\\
	&\geq-\beta\sum_{j=1}^{N}\hat{p}_j\log(\hat{p}_j)\\
	&\geq-\beta\sum_{j=1}^{N}\hat{p}_j\log(\hat{p}_n)\\
	&=-\beta\log(\hat{p}_n)\,,
	\end{align}
	where $\hat{p}_n$ is the largest value in $\{\hat{p}_1,\hat{p}_2,\dots,\hat{p}_N\}$. Then, from $\tilde{p}_n=\exp\left(-\frac{\mathcal{L}}{\alpha}\right)\hat{p}_n^{1-\frac{\beta}{\alpha}}$ and $\mathcal{L}\geq -\beta\log(\hat{p}_n)$, we have
	\begin{equation}
	\tilde{p}_n
	=\exp\left(-\frac{\mathcal{L}}{\alpha}\right)\hat{p}_n^{1-\frac{\beta}{\alpha}}
	\leq\exp\left(\frac{\beta\log(\hat{p}_n)}{\alpha}\right)\hat{p}_n^{1-\frac{\beta}{\alpha}}
	=\hat{p}_n\,.
	\end{equation}
	\qed
\end{proof}

\begin{figure}
	\centering	
	\includegraphics[width=0.9\linewidth]{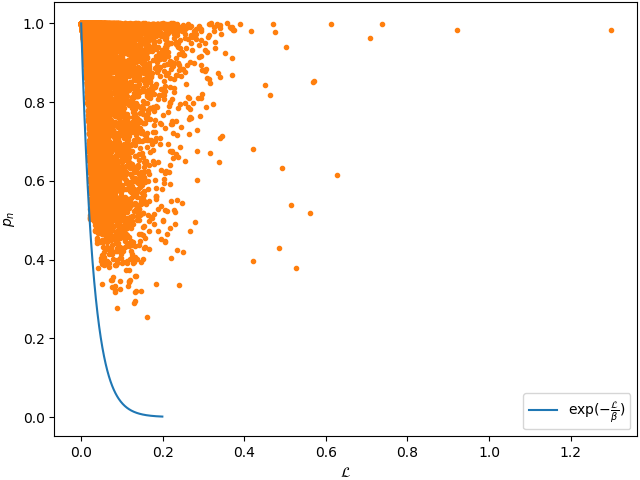}
	\caption{The curve of $\exp(-\frac{\mathcal{L}}{\beta})$ and $(\mathcal{L}, p_n)$ at the end of the D2 training on CIFAR-10. This figure is best viewed in color.}
	\label{fig:p-loss-data}
\end{figure}

We show that $\tilde{p}_n\leq \hat{p}_n$ holds in experiments. With $\tilde{p}_n=\exp(-\frac{\mathcal{L}}{\alpha})\hat{p}_n^{1-\frac{\beta}{\alpha}}$, if $\tilde{p}_n\leq \hat{p}_n$, that yields $\exp(-\frac{\mathcal{L}}{\alpha})\hat{p}_n^{1-\frac{\beta}{\alpha}}\leq \hat{p}_n$. Then we can get $\hat{p}_n\geq\exp(-\frac{\mathcal{L}}{\beta})$. 
Figure~\ref{fig:p-loss-data} shows $\hat{p}_n$ versus $\exp(-\frac{\mathcal{L}}{\beta})$, in which $\beta=0.03$.  
For a specific loss value, if $p_n$ is above the function curve, $\tilde{p}_n$ is smaller than $p_n$. Figure~\ref{fig:p-loss-data} shows the scatter plot of $(\mathcal{L}, \hat{p}_n)$ at the end of the D2 training on CIFAR-10.
Almost all points are above the curve. That means if $\tilde{p}_n\rightarrow\exp(-\frac{\mathcal{L}}{\alpha})\hat{p}_n^{1-\frac{\beta}{\alpha}}$, $\tilde{p}_n$ will be smaller than $\hat{p}_n$. 

From Theorem~\ref{thm:1}, we get $\tilde{p}_n\rightarrow\exp(-\frac{\mathcal{L}}{\alpha})\left(\hat{p}_n\right)^{1-\frac{\beta}{\alpha}}$,  where $\hat{\mathbf{p}}$ gets the largest value at $\hat{p}_n$. And Theorem~\ref{thm:2} tells us if $\tilde{p}_n=\exp(-\frac{\mathcal{L}}{\alpha})\left(\hat{p}_n\right)^{1-\frac{\beta}{\alpha}}$ then $\tilde{p}_n$ will be smaller than $\hat{p}_n$. Because $\tilde{\mathbf{p}}$ and $\hat{\mathbf{p}}$ are probability distributions, if $\tilde{\mathbf{p}}$ and $\hat{\mathbf{p}}$ get their largest value at $n$, $\tilde{\mathbf{p}}$ is more flat than $\hat{\mathbf{p}}$ when $\tilde{p}_n\leq \hat{p}_n$. That is, along with the training of D2, there is a tendency that pseudo-labels will be more flat than the network predictions.

Second, we find an unsolicited bias in the D2 framework. From the updating formula, we can get 
\begin{align}
\sum_{i=1}^{N}\tilde{y}_i
&\leftarrow
\sum_{i=1}^{N}\tilde{y}_i
-\lambda\alpha\sum_{i=1}^{N}\sigma\left(\tilde{\mathbf{y}}\right)_i
+\lambda\alpha\sum_{i=1}^{N}\sigma\left(\hat{\mathbf{y}}\right)_i  \\
&=\sum_{i=1}^{N}\tilde{y}_i-\lambda\alpha+\lambda\alpha  \\
&=\sum_{i=1}^{N}\tilde{y}_i\,.
\end{align}
That is, $\sum_{i=1}^{N}\tilde{y}_i$ will \emph{not} change after initialization. Although we define $\tilde{\mathbf{y}}$ as the variable which is unconstrained, the softmax function and SGD set an equality constraint for it. On the other hand, in practice, $\sum_{i=1}^{N}\hat{y}_i$ become more and more concentrated (cf. Figure~\ref{fig:stage2-y-sum}). Later, we will use an ablation study to demonstrate this bias is harmful.

\begin{figure}
	\centering 
	\subfloat[]{\label{fig:stage2-y-sum:sub1}\includegraphics[width=0.33\linewidth]{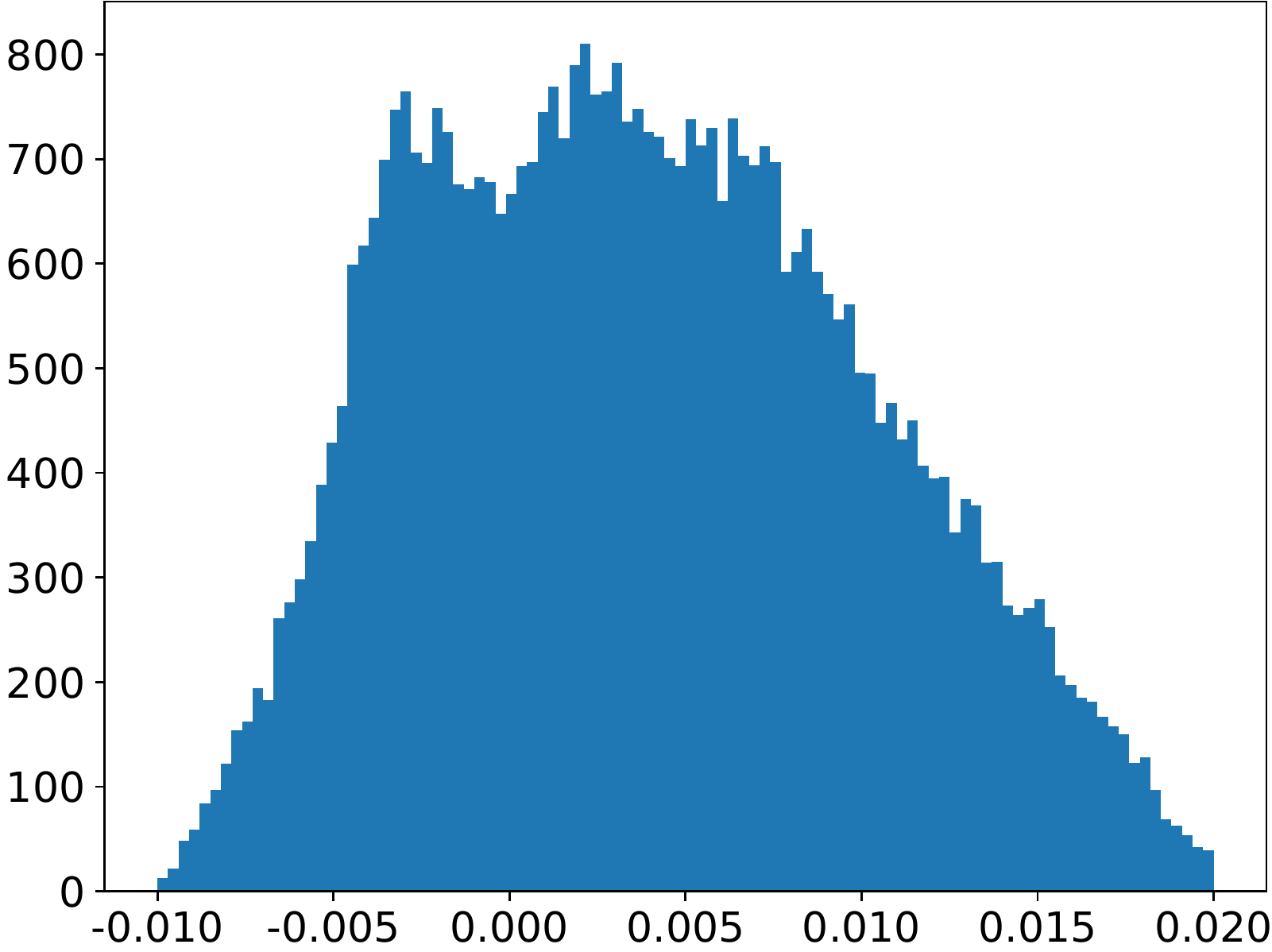}}
	\subfloat[]{\label{fig:stage2-y-sum:sub2}\includegraphics[width=0.33\linewidth]{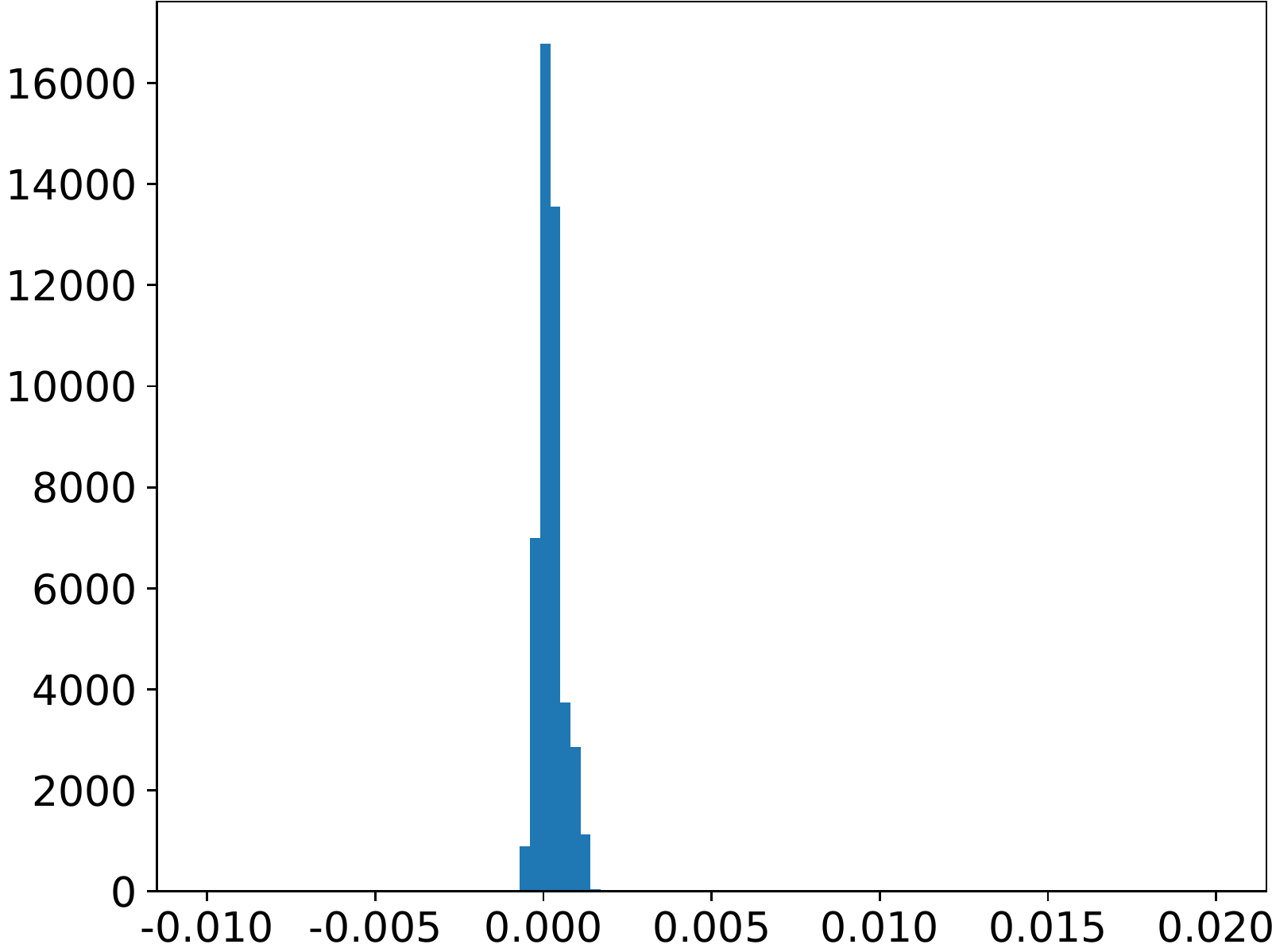}}
	\subfloat[]{\label{fig:stage2-y-sum:sub3}\includegraphics[width=0.33\linewidth]{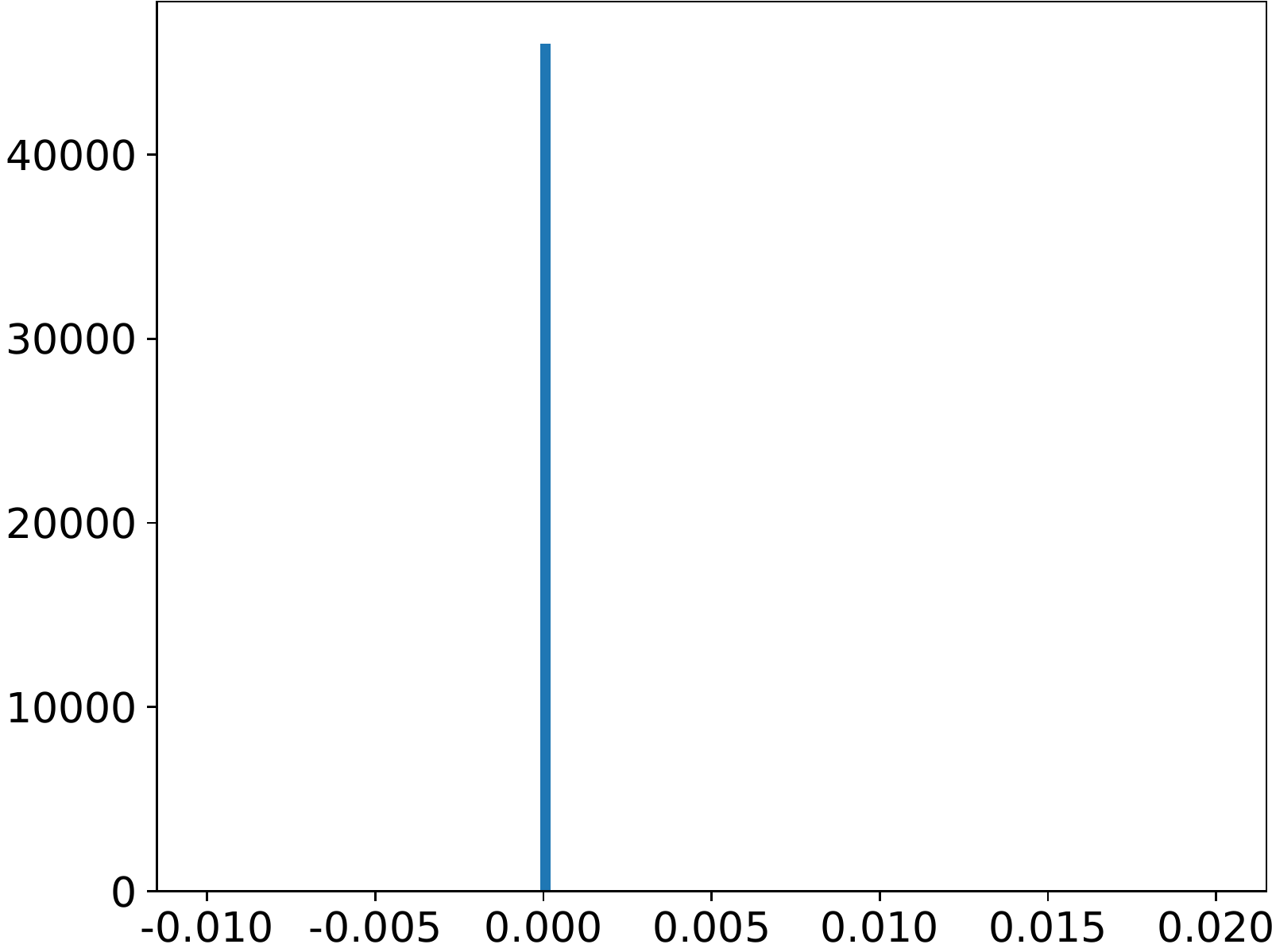}}\\ 
	\subfloat[]{\label{fig:stage2-y-sum:sub4}\includegraphics[width=0.33\linewidth]{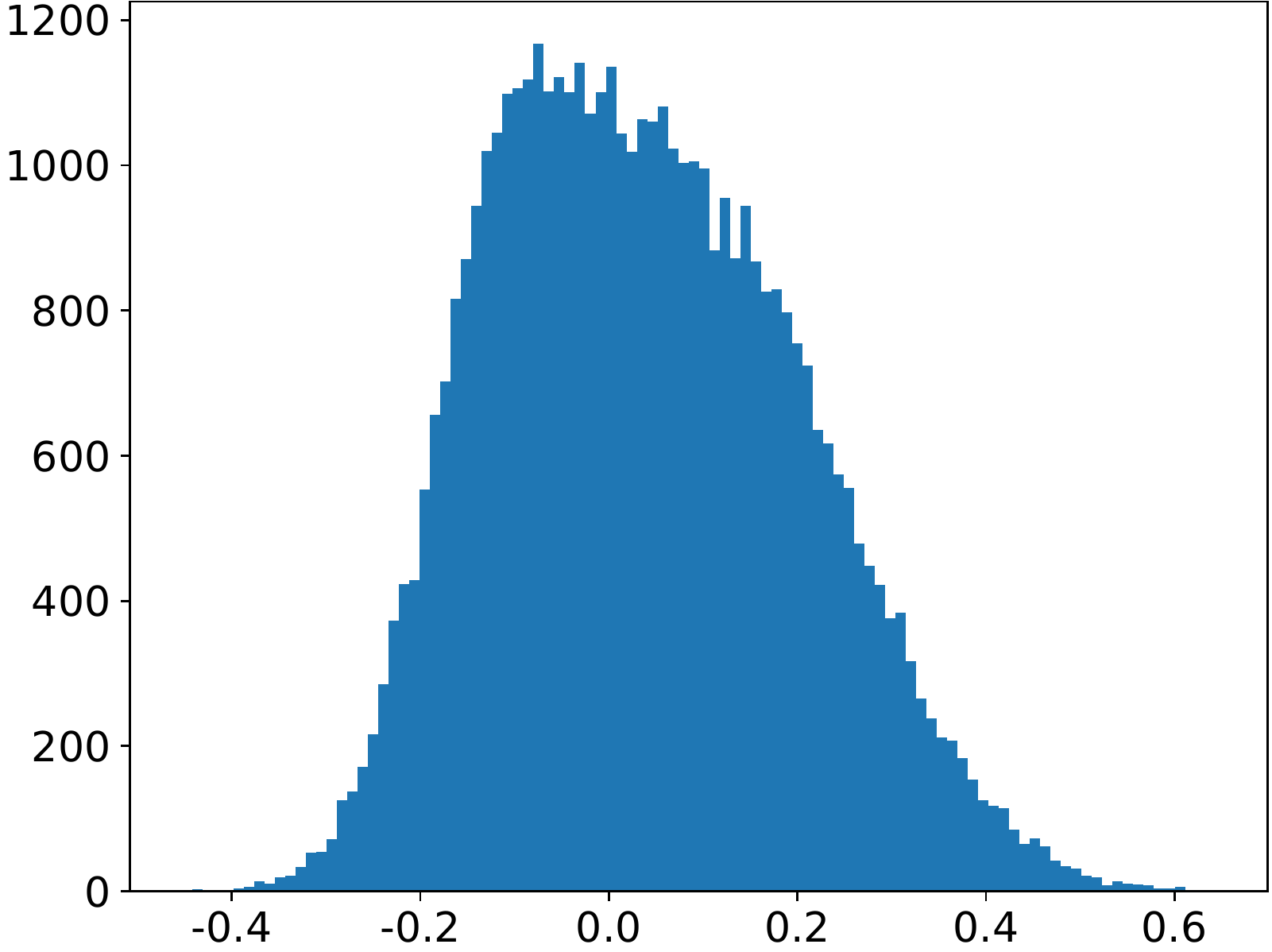}}
	\subfloat[]{\label{fig:stage2-y-sum:sub5}\includegraphics[width=0.33\linewidth]{fig_D2-pseudo.pdf}}
	\subfloat[]{\label{fig:stage2-y-sum:sub6}\includegraphics[width=0.33\linewidth]{fig_D2-pseudo.pdf}}\\ 
	\caption{During the D2 training, the distribution of $\sum_{i=1}^{N}\hat{y}_i$ on the unlabeled samples at 100, 200, 300 epoch on CIFAR-10 are showed by \protect\subref{fig:stage2-y-sum:sub1}, \protect\subref{fig:stage2-y-sum:sub2}, \protect\subref{fig:stage2-y-sum:sub3}, respectively. Note that $\sum_{i=1}^{N}\hat{y}_i$ will get more and more concentrated with training. The distributions of $\sum_{i=1}^{N}\tilde{y}_i$ on the unlabeled samples at 100, 200, 300 epoch on CIFAR-10 are showed by \protect\subref{fig:stage2-y-sum:sub4}, \protect\subref{fig:stage2-y-sum:sub5}, \protect\subref{fig:stage2-y-sum:sub6}, respectively. According to our analysis, $\sum_{i=1}^{N}\tilde{y}_i$ will not change. Experimental results are consistent with our theoretical analysis.}
	\label{fig:stage2-y-sum}
\end{figure}

We propose a repetitive reprediction (R2) strategy to overcome these difficulties, which repeatedly perform repredictions (i.e., using the prediction $\tilde{\mathbf{y}}$ to re-initialize the pseudo-labels $\tilde{\mathbf{y}}$ several times) during training D2. The benefits of R2 are two-fold. First, we want to make pseudo-labels confident. According to our analysis, the network predictions are sharper than pseudo-labels when the algorithm converges. So repredicting pseudo-labels can make them sharper.
Second, $\sum_{i=1}^{N}\tilde{y}_i$ will not change during D2 training. Reprediction can reduce the impact of this bias.
Furthermore, the validation accuracy often increase during training. A repeated reprediction can make pseudo-labels more accurate than that of the last reprediction. 

\begin{figure}[!htb] 
	\centering 
	\subfloat[]{\label{fig:stage2-loss:sub1}\includegraphics[width=0.3\linewidth]{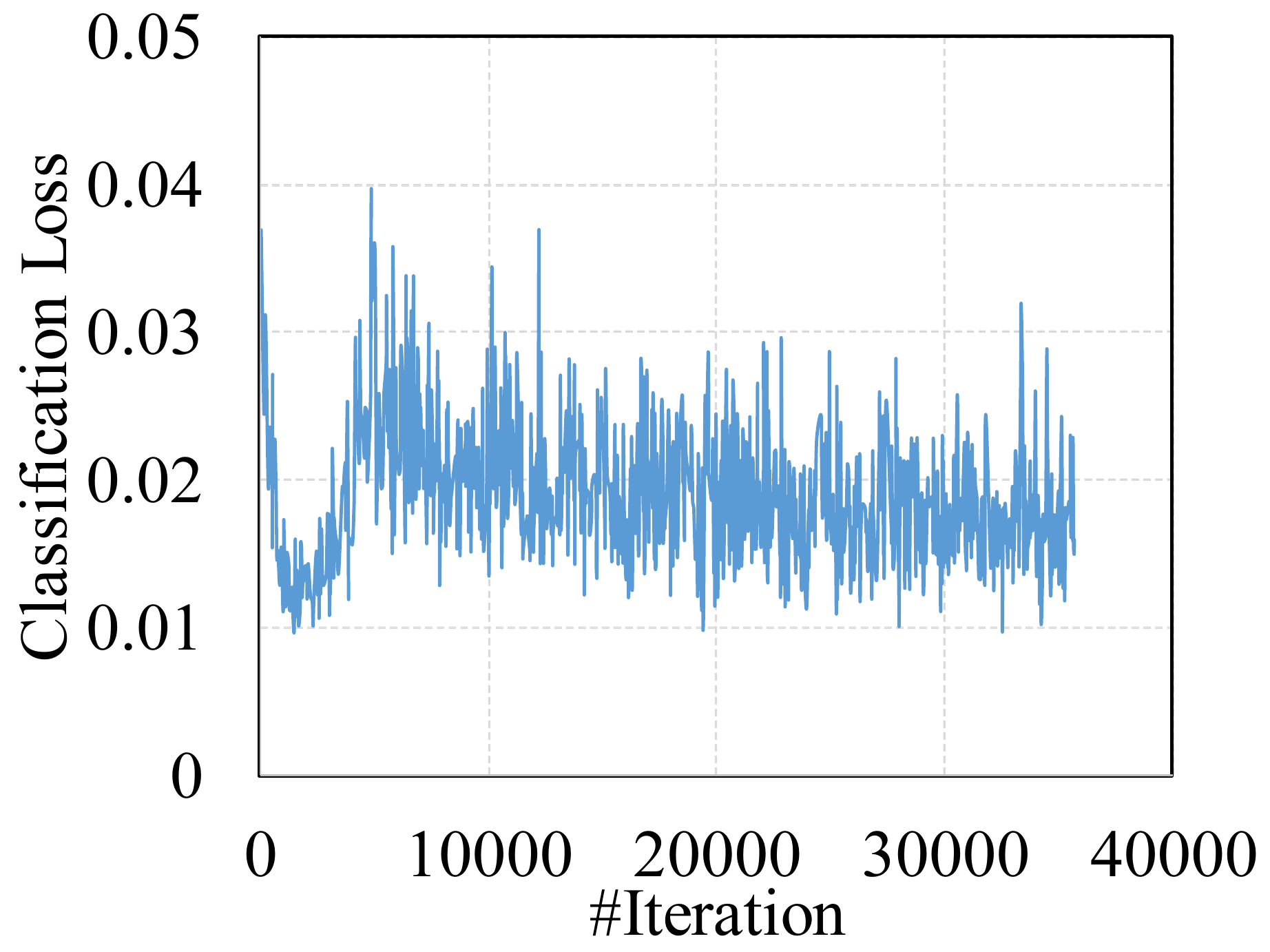}}
	\subfloat[]{\label{fig:stage2-loss:sub2}\includegraphics[width=0.3\linewidth]{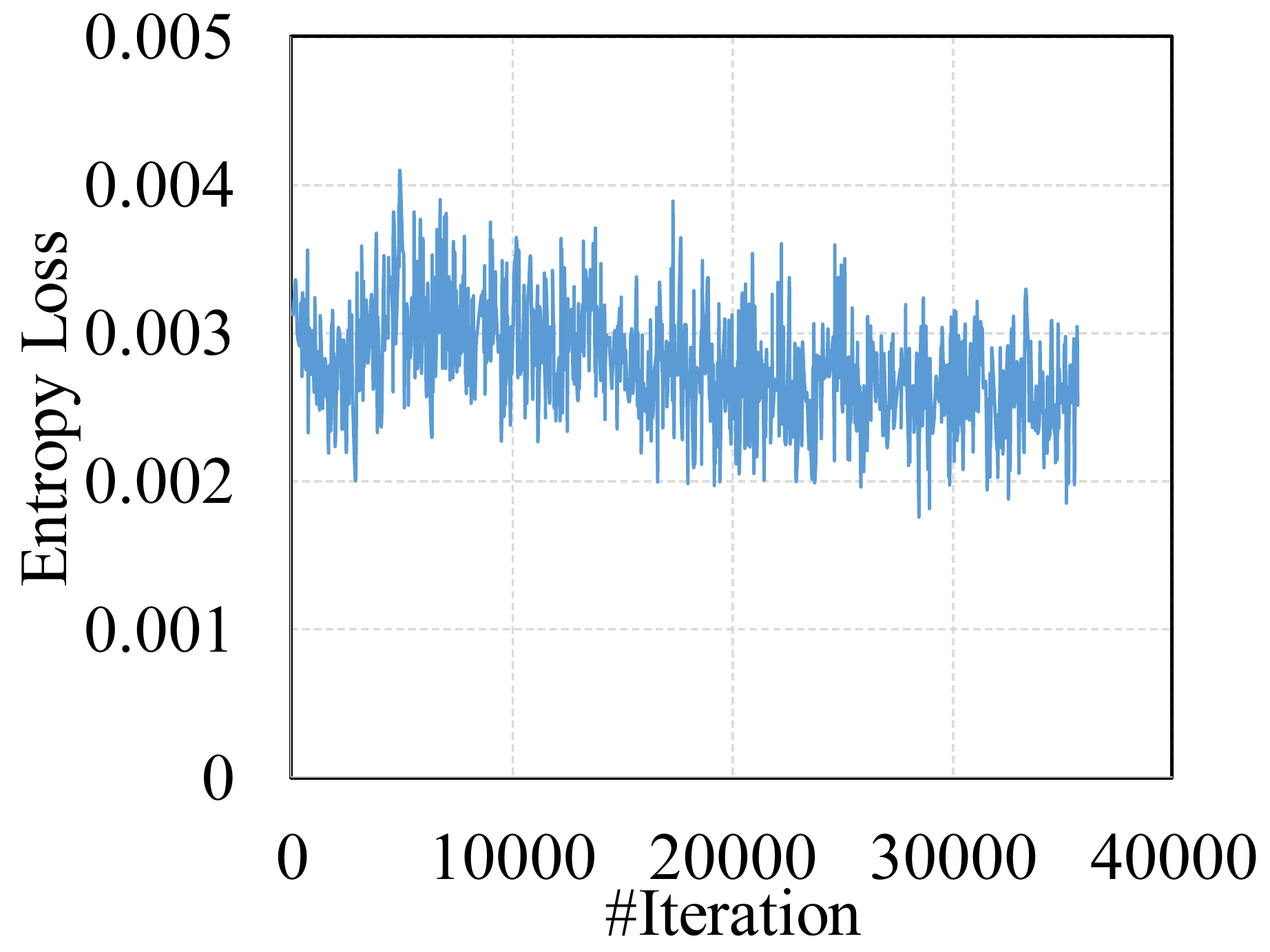}}
	\subfloat[]{\label{fig:stage2-loss:sub3}\includegraphics[width=0.3\linewidth]{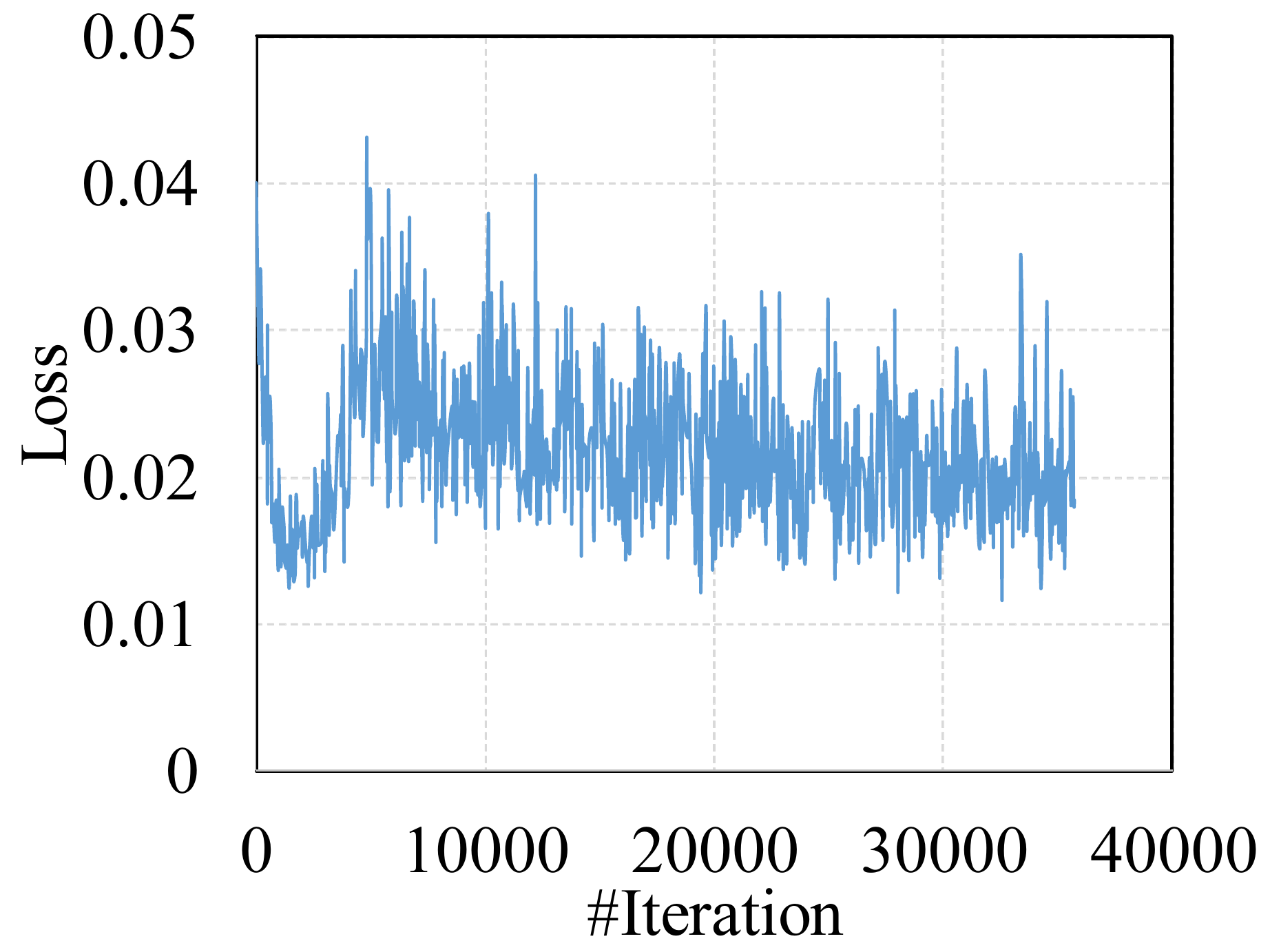}} 
	\caption{\protect\subref{fig:stage2-loss:sub1}, \protect\subref{fig:stage2-loss:sub2}, \protect\subref{fig:stage2-loss:sub3} show how $\mathcal{L}_c$, $\mathcal{L}_e$, $\mathcal{L}$ change by training D2 without R2 on CIFAR-10 (column a in Table~\ref{tab:ablation-result-table}), respectively. With a fixed learning rate, it is difficult for these loss terms to decrease.}
	\label{fig:stage2-loss} 
\end{figure}
\begin{figure}[!htb] 
	\centering 
	\subfloat[]{\label{fig:stage3-loss:sub1}\includegraphics[width=0.3\linewidth]{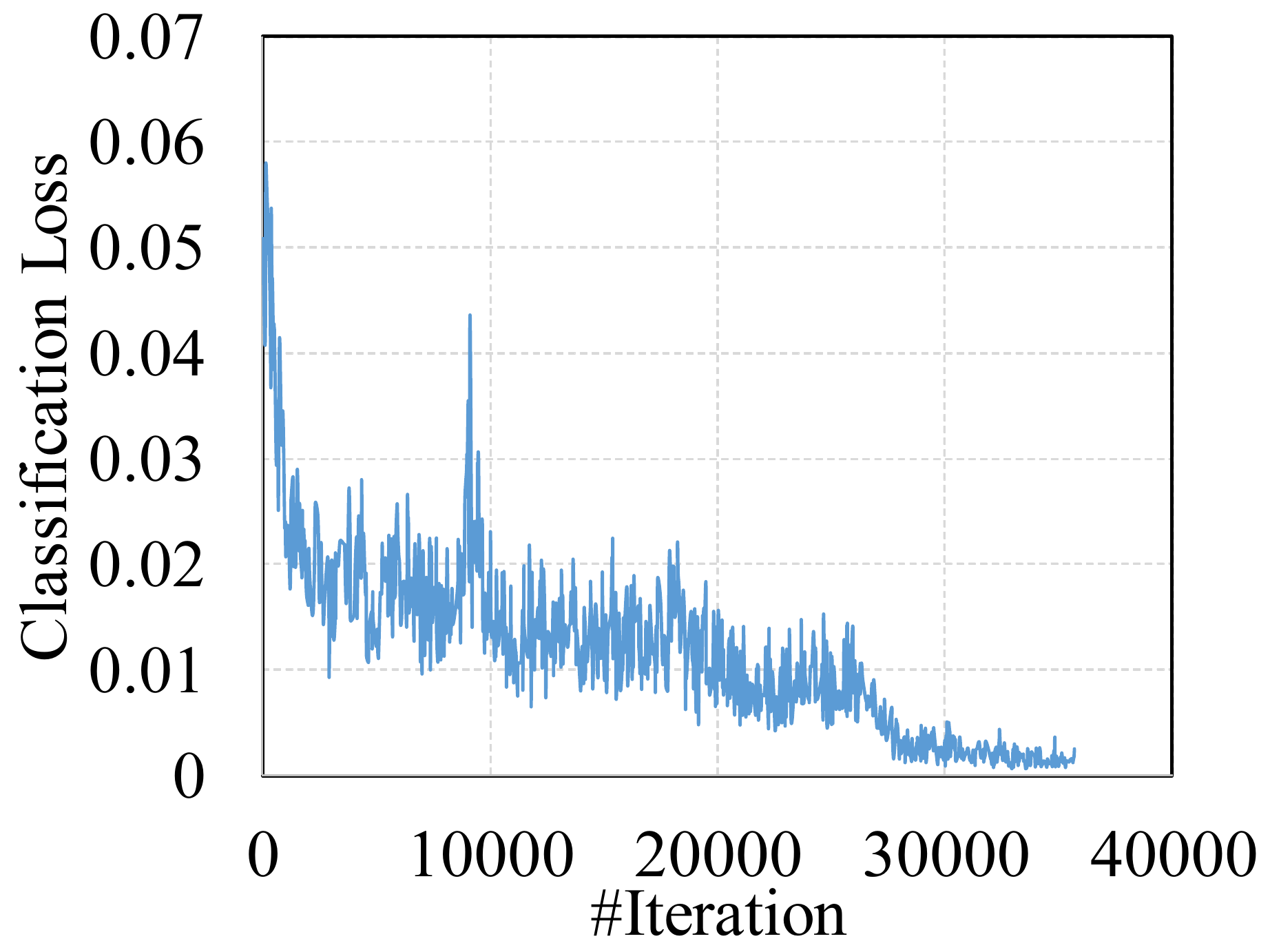}}
	\subfloat[]{\label{fig:stage3-loss:sub2}\includegraphics[width=0.3\linewidth]{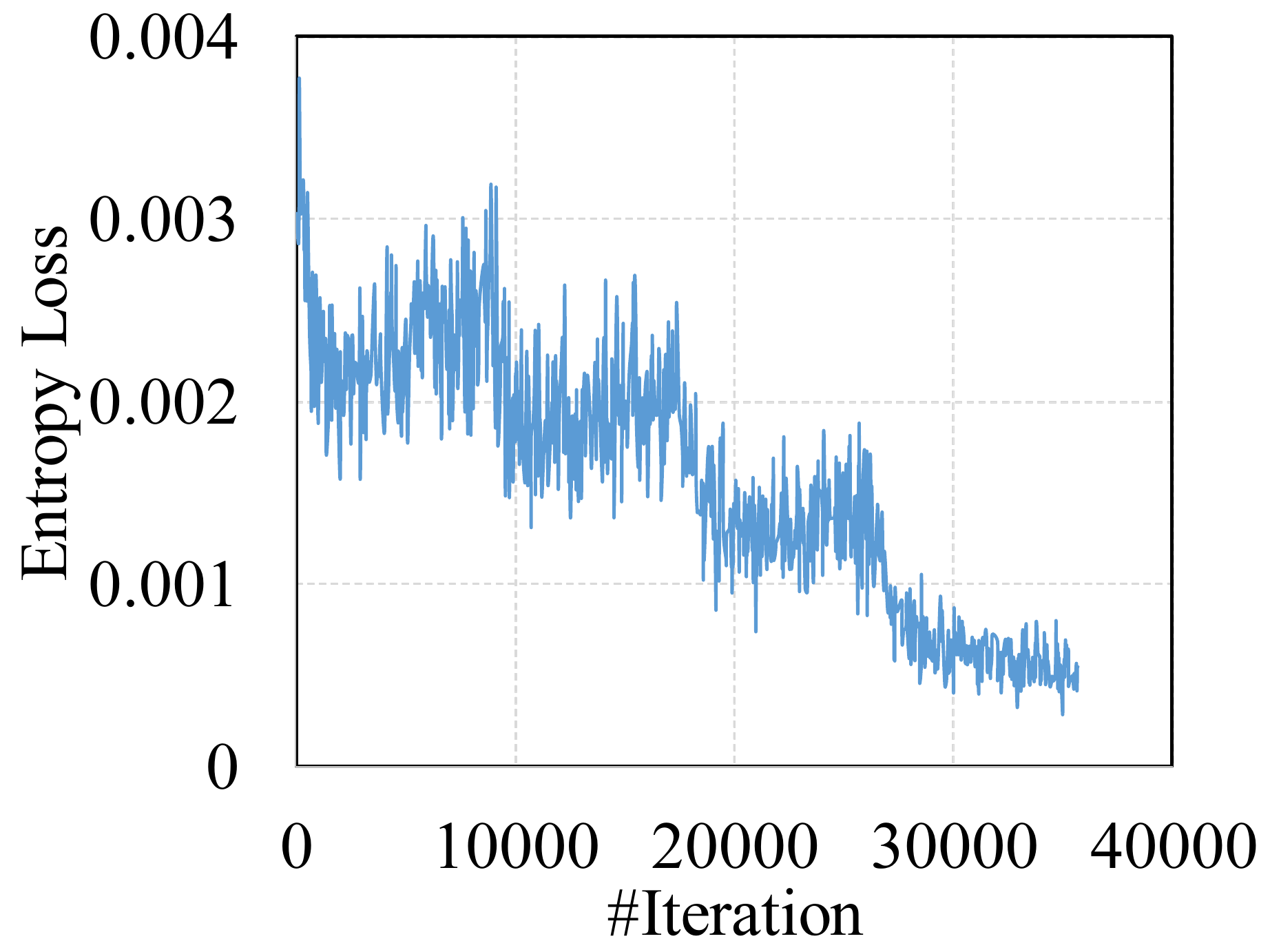}}
	\subfloat[]{\label{fig:stage3-loss:sub3}\includegraphics[width=0.3\linewidth]{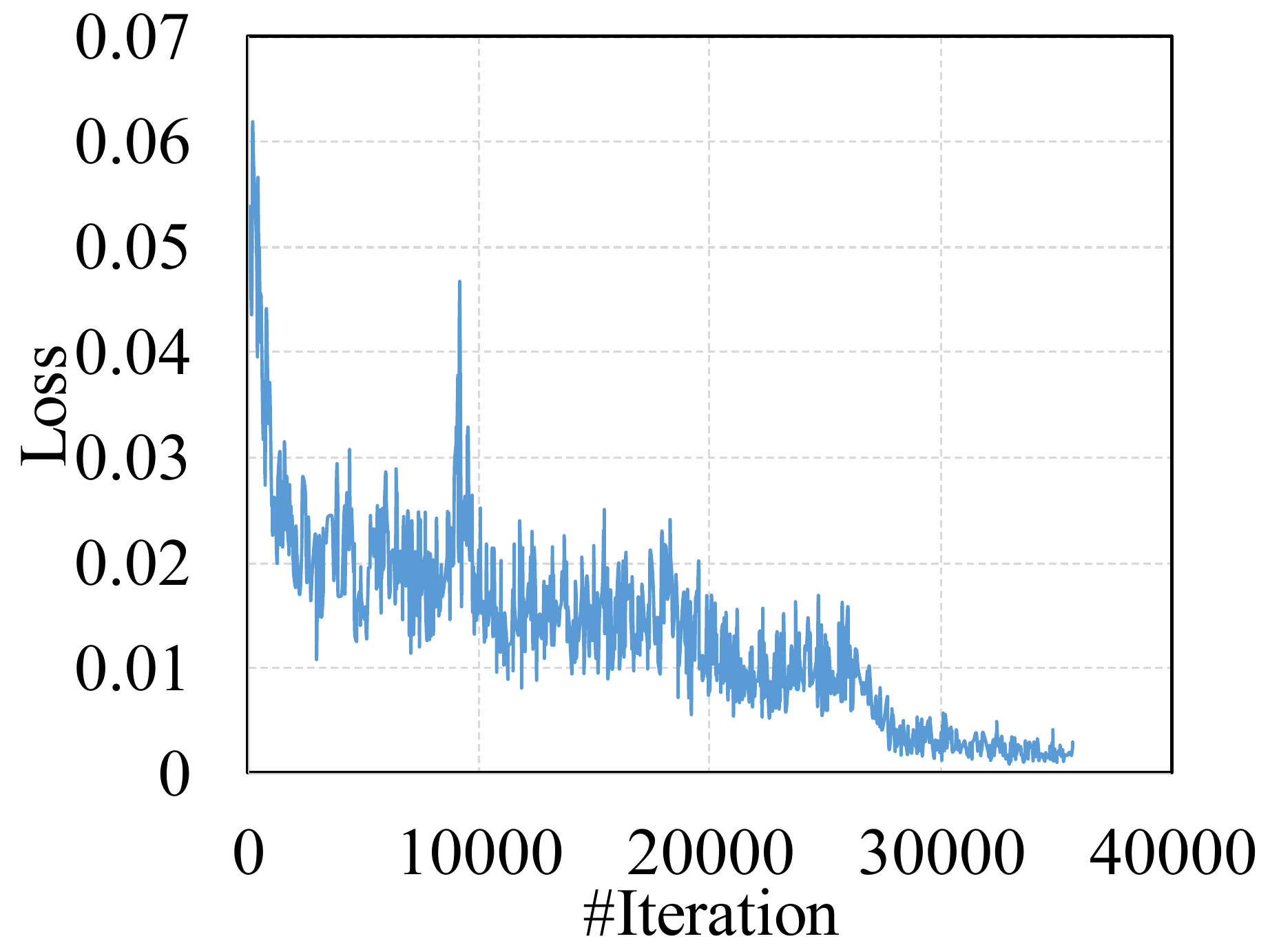}} 
	\caption{\protect\subref{fig:stage3-loss:sub1}, \protect\subref{fig:stage3-loss:sub2}, \protect\subref{fig:stage3-loss:sub3} show how $\mathcal{L}_c$, $\mathcal{L}_e$, $\mathcal{L}$ change by training D2 with R2 on CIFAR-10 (column e in Table~\ref{tab:ablation-result-table}), respectively. Reprediction occurs at 8750, 17500, and 26250 iterations. After each reprediction, we decrease the learning rate.}
	\label{fig:stage3-loss} 
\end{figure}

Apart from the repredictions, we also reduce the learning rate to boost the performance. If the D2 framework is trained by a fixed learning rate as in \citet{PENCIL}, the loss $\mathcal{L}$ did not descend in experiments (cf. Figure~\ref{fig:stage2-loss}). Reducing the learning rate can make the loss descend (cf. Figure~\ref{fig:stage3-loss}). We can get some benefits from a lower loss. On one hand, $\mathcal{L}_c$ is the KL divergence between pseudo-labels and the network predictions. Minimizing this term makes pseudo-labels as sharp as the network predictions. On the other hand, minimizing $\mathcal{L}_e$ can decrease the entropy of network predictions. So when it comes to the next reprediction, pseudo-labels will be more confident according to sharper predictions. 

\begin{figure}
	\centering	
	\includegraphics[width=0.9\linewidth]{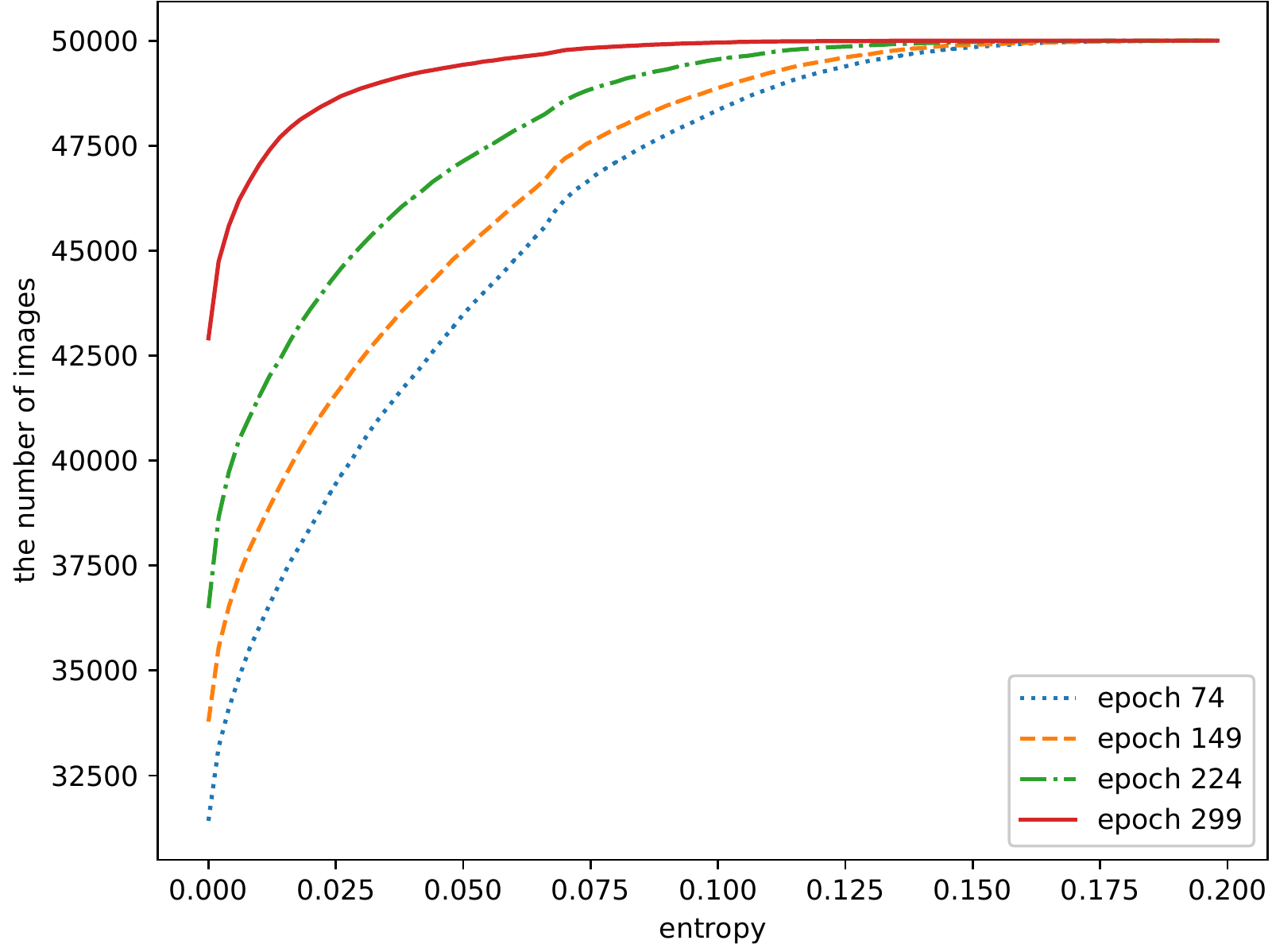}
	\caption{Cumulative distribution of the number of pseudo-labels versus the entropy after using the repetitive reprediction (R2) strategy on CIFAR-10. For each point $(n, e)$ on the line, it means there are $n$ images whose label entropies are less than $e$. Using the R2 strategy can make pseudo-labels sharper at the end of training. This figure is best viewed in color.}
	\label{fig:stage3-num_entropy}
\end{figure}

Finally, repredicting pseudo-labels frequently is harmful for performance. By using the R2 strategy every epoch, the network predictions and pseudo-labels are always the same and D2 cannot optimize pseudo-labels anymore. In CIFAR-10 experiments, we repredict pseudo-labels every 75 epochs and reduce the learning rate after each reprediction. Figure~\ref{fig:stage3-num_entropy} shows that using the R2 strategy can make pseudo-labels more confident at the end of training.

\subsection{The overall R2-D2 algorithm}

Now we propose the overall R2-D2 algorithm. The training can be divided into three stages. In the first stage, we only use labeled images to train the backbone network with cross entropy loss as in common network training. In the second stage, we use the backbone network trained in the first stage to predict pseudo-labels for unlabeled images. Then we use D2 to train the network and optimize pseudo-labels together. It is expected that this stage can boost the network performance and make pseudo-labels more precise. But according to our analysis, it is not enough to train D2 by only one stage. With the R2 strategy, D2 will be repredicted and trained for several times. In the third stage, the backbone network is finetuned by all images whose labels come from the second stage. For unlabeled images, we pick the class which has the maximum value in pseudo-labels and use the cross entropy loss to train the network. And pseudo-labels are not updated anymore. For labeled images, we use their groundtruth labels.

In general, R2-D2 is a simple method. It requires only one single network (versus two in Mean Teacher) and the loss function consists of two terms (versus three in Mean Teacher). The training processes in different stages are identical (share the same code), just need to change the value of two switch variables.

\section{Experiments}

In this section, we use four datasets to evaluate our algorithm: ImageNet~\citep{imagenet}, CIFAR-100~\citep{cifar}, CIFAR-10~\citep{cifar}, SVHN~\citep{SVHN}. We first use an ablation study to investigate the impact of the R2 strategy. We then report the results on these datasets to compare with state-of-the-arts. We also conduct experiments that use R2-D2 to finetune other SSL methods. At last, we evaluate R2-D2 under the realistic setting. All experiments were implemented using the PyTorch~\citep{pytorch} framework and run on a computer with TITAN Xp GPU.

\subsection{Implementation details}

Note that we trained the network using stochastic gradient descent with Nesterov momentum 0.9 in all experiments. We set $\alpha = 0.1$, $\beta=0.03$ and $\lambda=4000$ on \emph{all} datasets, which shows the robustness of our method to these hyperparameters. Other hyperparameters (e.g., batch size, learning rate, and weight decay) were set according to different datasets. 

\textbf{ImageNet} is a large-scale dataset with natural color images from 1000 categories. Each category typically has 1300 images for training and 50 for evaluation. Following the prior work~\citep{DCT_2018_ECCV,Stochastic_Transformations,VAE,Mean_teacher}, we uniformly choose 10\% data from training images as labeled data. That means there are 128 labeled data for each category. The rest of training images are considered as unlabeled data. We test our model on the validation set. The backbone network is ResNet-18~\citep{ResNet}. %More details can be found in Appendix~\ref{apx:ImageNet}.
The data augmentation we used is the same as that of \citet{Mean_teacher}, which includes random rotation, random resized crop to $224\times 224$, random horizontal flip and color jittering. 

In the first stage, we trained ResNet-18~\citep{ResNet} on 4 GPUs with the labeled data. We trained for 60 epochs with the weight decay of $5\times 10^{-5}$. Because the labeled dataset only contains 128000 images, the batch size was set as 160 to make the parameters update more times.  The learning rate was 0.1 at the beginning and decreased by cosine annealing~\citep{cosine_lr} so that it would reach 0 after 75 epochs. 

In the second stage, we trained for 60 epochs on 4 GPUs. We set the batch size as 800, 200 of which were labeled. The learning rate was 0.12 and did not change in this stage. During this stage, we found that pseudo-labels would be more accurate. Note that the capacity of ResNet-18 is small and it is hard for ResNet-18 to remember all pseudo-labels. To make pseudo-labels more accurate when repredicting, we finetuned the model using the dataset with pseudo-label. We finetuned for 60 epochs with initial learning rate 0.12 and decayed it with cosine annealing~\citep{cosine_lr} so that is would reach 0 after 65 epochs.

Repeating the second stage, we used the network at the end of last stage to repridict the pseudo-labels of unlabeled images. Then we trained the network and optimize pseudo-labels for 30 epochs with learning rate 0.04. Other settings were the same as the second stage.

In the third stage, we used the pseudo-labels at the end of last stage to finetune the model. We finetuned for 60 epochs with initial learning rate 0.04 and decayed it with cosine annealing~\citep{cosine_lr} so that is would reach 0 after 65 epochs.

\textbf{CIFAR-100} contains $32\times 32$ natural images from 100 categories. There are 50000 training images and 10000 testing images in CIFAR-100. Following \citet{Temporal_Ensembling,DCT_2018_ECCV,cvpr2019_pseudo_label}, we use 10000 images (100 per class) as labeled data and the rest 40000 as unlabeled data. We report the error rates on the testing images. The backbone network is ConvLarge~\citep{Temporal_Ensembling}. %More details can be found in Appendix~\ref{apx:CIFAR-100}.  
The data augmentation contained random translations, random horizontal flip and cutout~\citep{cutout}. 

In the first stage, we trained the ConvLarge network on 1 GPU with 10000 labeled images. To make the parameters update more times, we set the batch size as 20 and trained the network for 300 epochs. So the parameters of the network could update 500 times per epoch and update 150000 times totally in this stage. The initial learning rate was 0.05 and decreased by cosine annealing~\citep{cosine_lr} so that it would reach 0 after 350 epochs. The weight decay was set as 0.0002.

In the second stage, we optimized the network and pseudo-labels for 300 epochs on 4 GPUs. The batch size was 512, in which 128 images were labeled and others were unlabeled. The learning rate was 0.04 and did not change in this stage.

Repeating the second stage, we repredicted pseudo-labels at 0, 75, 150, 225 epoch. After each reprediction, we optimized the network and pseudo-labels for 75 epochs. The learning rate were set as 0.04, 0.03, 0.02, 0.01, respectively. Other settings were the same as the second stage.

In the third stage, we finetuned the network for 50 epochs with batch size 512. The learning rate was 0.01 at the beginning and decreased by cosine annealing~\citep{cosine_lr} so that it would reach 0 at the end.

\textbf{CIFAR-10} contains $32\times 32$ natural images from 10 categories. Following \citet{Temporal_Ensembling,VAT,Mean_teacher,DCT_2018_ECCV,HybridNet}, we use 4000 images (400 per class) from 50000 training images as labeled data and the rest images as unlabeled data. We report the error rates on the full 10000 testing images. The backbone network is Shake-Shake~\citep{shakeshake}. %More details can be found in Appendix~\ref{apx:CIFAR-10}.
The data augmentation contained random translations, random horizontal flip and cutout~\citep{cutout}.

All the settings were the same with that of CIFAR-100 except the learning rate and batch size. In the first stage, we trained the Shake-Shake network on 1 GPU with 4000 labeled images. We set the batch size as 40 and trained the network for 300 epochs. The initial learning rate was 0.05 and decreased by cosine annealing~\citep{cosine_lr} so that it would reach 0 after 350 epochs. The weight decay was set as 0.0002.

In the second stage, we optimized the network and pseudo-labels for 300 epochs on 4 GPUs. The batch size was 512, in which 128 images were labeled and others were unlabeled. The learning rate was 0.12 and did not change in this stage.

Repeating the second stage, we repredicted pseudo-labels at 0, 75, 150, 225 epoch. After each reprediction, we optimized the network and pseudo-labels for 75 epochs. The learning rate were set as 0.12, 0.08, 0.04, 0.004, respectively. Other settings were the same as the second stage.

In the third stage, we finetuned the network for 50 epochs with batch size 512. The learning rate was 0.01 at the beginning and decreased by cosine annealing~\citep{cosine_lr} so that it would reach 0 at the end.

\textbf{SVHN} dataset consists of $32\times 32$ house number images belonging to 10 classes. The category of each image is the centermost digit. There are 73257 training images and 26032 testing images in SVHN. Following \citet{Temporal_Ensembling,Mean_teacher,VAT,DCT_2018_ECCV}, we use 1000 images (100 per class) as labeled data and the rest 72257 training images as unlabeled data. The backbone network is ConvLarge~\citep{Temporal_Ensembling}. %More details can be found in Appendix~\ref{apx:SVHN}.
The data augmentation consists of adding gaussian noise to images like~\citet{Temporal_Ensembling,Mean_teacher} and cutout~\citep{cutout}.

The settings of learning rates and weight decay were the same as that of our training strategy for CIFAR-10. In the first stage, we trained the ConvLarge~\citep{Temporal_Ensembling} network on 1 GPU for 180 epochs with batch size 10. In the second stage, the batch size was set as 512, in which 128 images were labeled. The network was trained for 180 epochs. Repeating the second stage, pseudo-labels were repredicted at 0, 45, 90, 135 epoch. In the third stage, we finetuned the network for 180 epochs.

\subsection{Ablation studies}

\begin{table}
	\caption{Ablation studies when using different strategies to train our end-to-end framework on CIFAR-10. ($\alpha=0.1,\beta=0.03,\lambda=4000$)}
	\label{tab:ablation-result-table}
	\centering
	\begin{tabular}{cccccc}
		\toprule
		& a & b & c & d & e \\
		\midrule
		The 2nd stage 			& $\checkmark$ 	& $\checkmark$  & $\checkmark$ & $\checkmark$ &$\checkmark$  \\
		Repeat the 2nd stage	&				& $\checkmark$  & $\checkmark$ & $\checkmark$ &$\checkmark$  \\
		Reprediction			&  				& 				& $\checkmark$ & 			 &$\checkmark$  \\
		Reducing LR 			&  				& 				&			  & $\checkmark$ &$\checkmark$  \\
		Error rates (\%) 		& 6.71 			& 6.37 			& 6.23		   & 5.94		 & 5.78 \\
		\bottomrule
	\end{tabular}
\end{table}

\begin{table}
	\caption{Ablation studies when using different $\alpha$ to train our end-to-end framework on CIFAR-10. ($\beta=0.03,\lambda=4000$)}
	\label{tab:ablation-alpha}
	\centering
	\begin{tabular}{cccccc}
		\toprule
		$\alpha$ & 0.1 & 0.2 & 0.3 & 0.4 & 0.5 \\
		\midrule
		Error rates (\%) & 5.78 & 5.44 & 5.81 & 5.90 & 6.11 \\
		\bottomrule
	\end{tabular}
\end{table}

\begin{table}
	\caption{Ablation studies when using different $\beta$ to train our end-to-end framework on CIFAR-10. ($\alpha=0.1,\lambda=4000$)}
	\label{tab:ablation-beta}
	\centering
	\begin{tabular}{cccccc}
		\toprule
		$\beta$ & 0.01 & 0.02 & 0.03 & 0.04 & 0.05 \\
		\midrule
		Error rates (\%) & 5.62 & 5.75 & 5.78 & 5.83 & 5.76 \\
		\bottomrule
	\end{tabular}
\end{table}

\begin{table}
	\caption{Ablation studies when using different $\lambda$ to train our end-to-end framework on CIFAR-10. ($\alpha=0.1,\beta=0.03$)}
	\label{tab:ablation-lambda}
	\centering
	\begin{tabular}{cccccc}
		\toprule
		$\lambda$ & 1000 & 2000 & 3000 & 4000 & 5000 \\
		\midrule
		Error rates (\%) & 5.85 & 5.85 & 5.82 & 5.78 & 5.53 \\
		\bottomrule
	\end{tabular}
\end{table}

\begin{table}
	\caption{Ablation studies when using different $\mathcal{L}_c$ to train our end-to-end framework on CIFAR-10. ($\alpha=0.1,\beta=0.03,\lambda=4000$)}
	\label{tab:ablation-loss}
	\centering
	\begin{tabular}{cccccc}
		\toprule
		$\mathcal{L}_c$ & $KL(\hat{\mathbf{p}}||\tilde{\mathbf{p}})$ & $KL(\tilde{\mathbf{p}}||\hat{\mathbf{p}})$ & $\|\tilde{\mathbf{p}} - \hat{\mathbf{p}}\|_2^2$ \\
		\midrule
		Error rates (\%) & 5.78 & 8.06 & 6.35 \\
		\bottomrule
	\end{tabular}
\end{table}

\begin{table*}
	\caption{Error rates (\%) on the validation set of ImageNet benchmark with 10\% images labeled. ``-'' means that the original papers did not report the corresponding error rates. ResNet-18 is used. }
	\label{tab:ImageNet-result-table}
	\centering
	\begin{tabular}{c|llccc}
		\toprule
		&
		Method	& Backbone		&	\#Param	&	Top-1	&	Top-5	\\
		\midrule
		\multirow{2}*{Supervised} &
		100\% Supervised			
		&	ResNet-18	&	11.6M	&	30.43	&  10.76	\\
		&
		10\% Supervised	
		&	ResNet-18	&	11.6M	&	52.23	&	27.54	\\
		\midrule
		\multirow{5}*{Semi-supervised} &
		Stochastic Transformations
		&	AlexNet		&	61.1M	&     -		& 39.84		\\
		&							
		VAE with 10\% Supervised
		&	Customized	&	30.6M	&	51.59	&	35.24	\\
		&							
		Mean Teacher				
		&	ResNet-18	&	11.6M	&	49.07	&  23.59	\\
		&							
		Dual-View Deep Co-Training
		&	ResNet-18	&	11.6M	&	46.50	&	22.73	\\	
		&
		R2-D2					
		&	ResNet-18	&	11.6M	&\textbf{41.55}&\textbf{19.52}\\
		\midrule
		\multirow{1}*{Self-supervised + Semi-supervised} &
		RotNet + R2-D2					
		&	ResNet-18	&	11.6M	&\textbf{40.54}&\textbf{18.76}\\
		\bottomrule
	\end{tabular}
\end{table*}

\begin{figure}
	\centering 
	\subfloat[]{\label{fig:ablation-c:sub1}\includegraphics[width=0.3\linewidth]{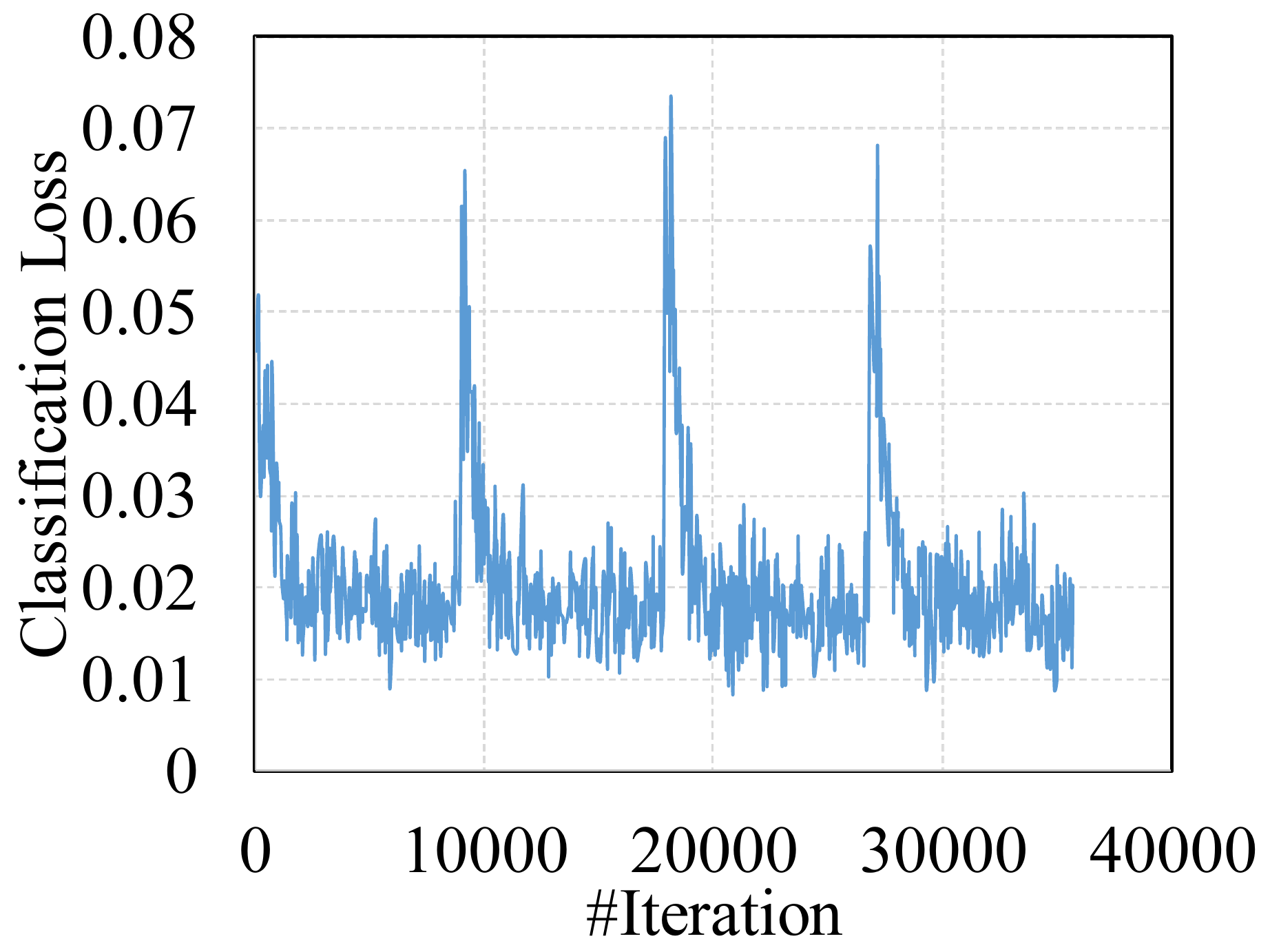}}
	\subfloat[]{\label{fig:ablation-c:sub2}\includegraphics[width=0.3\linewidth]{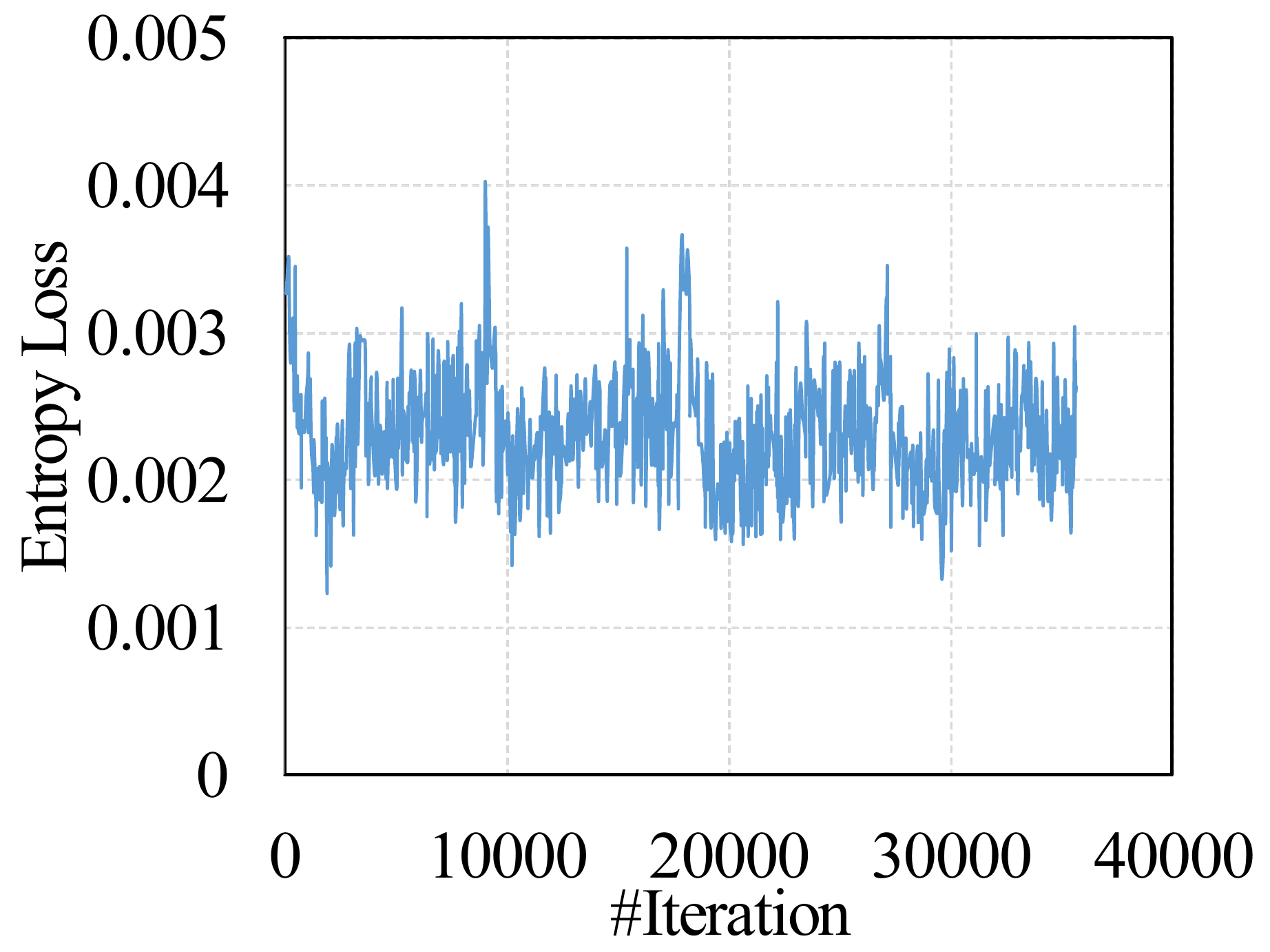}}
	\subfloat[]{\label{fig:ablation-c:sub3}\includegraphics[width=0.3\linewidth]{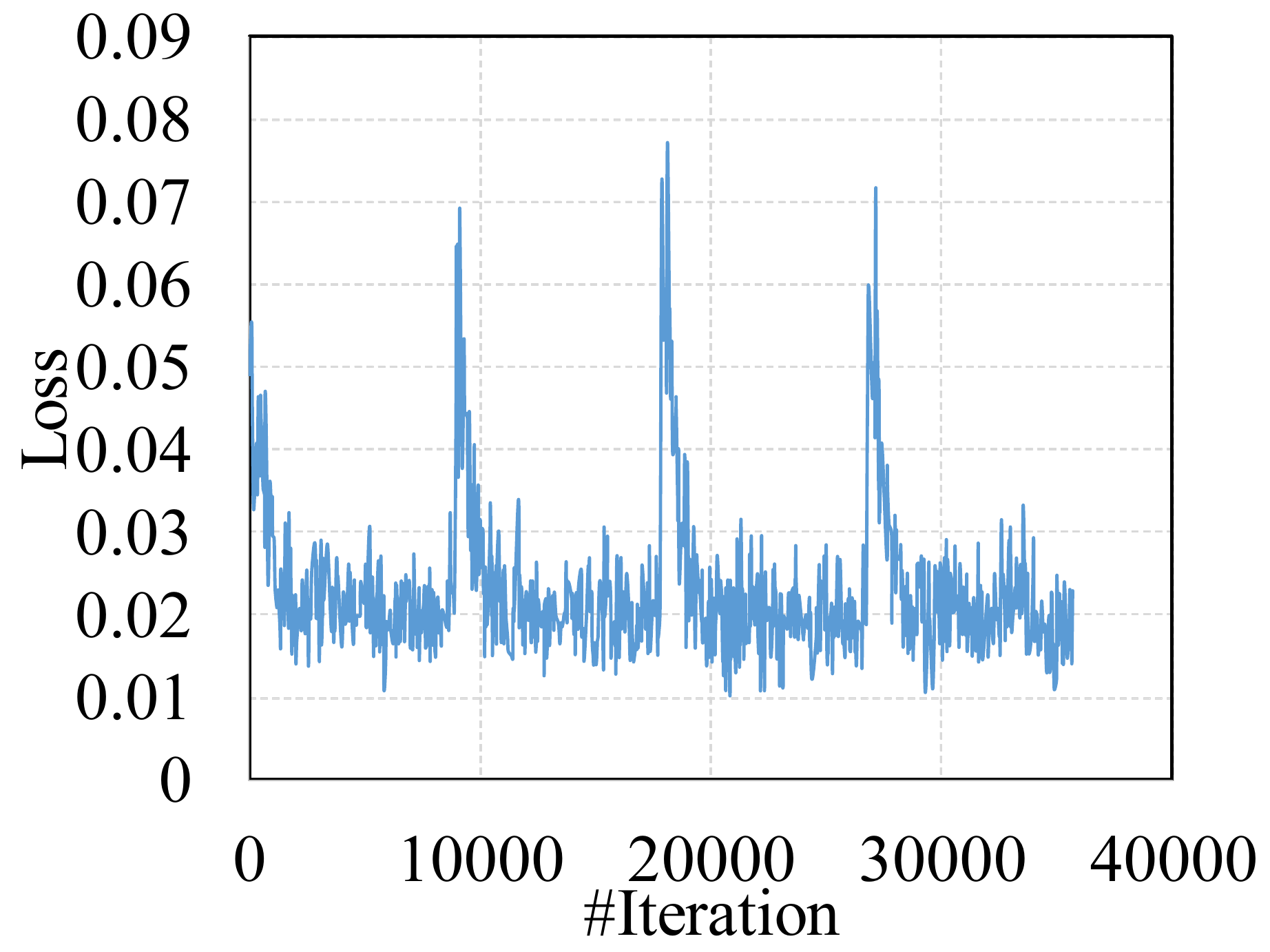}} 
	\caption{\protect\subref{fig:ablation-c:sub1}, \protect\subref{fig:ablation-c:sub2}, \protect\subref{fig:ablation-c:sub3} show how $\mathcal{L}_c$, $\mathcal{L}_e$, $\mathcal{L}$ change by training D2 with only reprediction on CIFAR-10 (column c in Table~\ref{tab:ablation-result-table}), respectively. Reprediction occurs at 8750, 17500, and 26250 iterations. Without reducing the learning rate, the loss cannot decrease and network prediction keep flat. So pseudo-labels will not become sharp.}
	\label{fig:ablation-c} 
\end{figure}
\begin{figure}
	\centering 
	\subfloat[]{\label{fig:ablation-d:sub1}\includegraphics[width=0.3\linewidth]{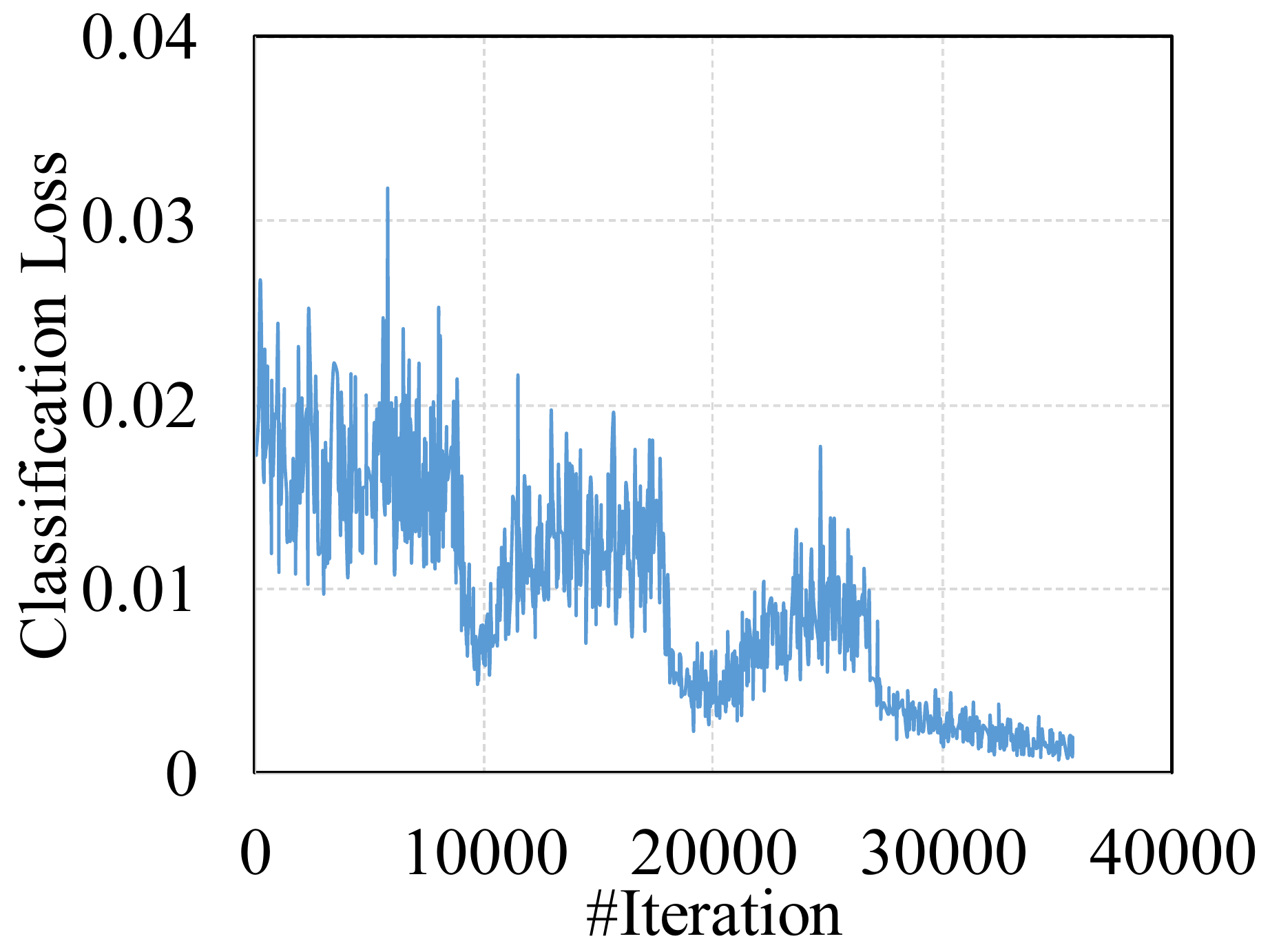}}
	\subfloat[]{\label{fig:ablation-d:sub2}\includegraphics[width=0.3\linewidth]{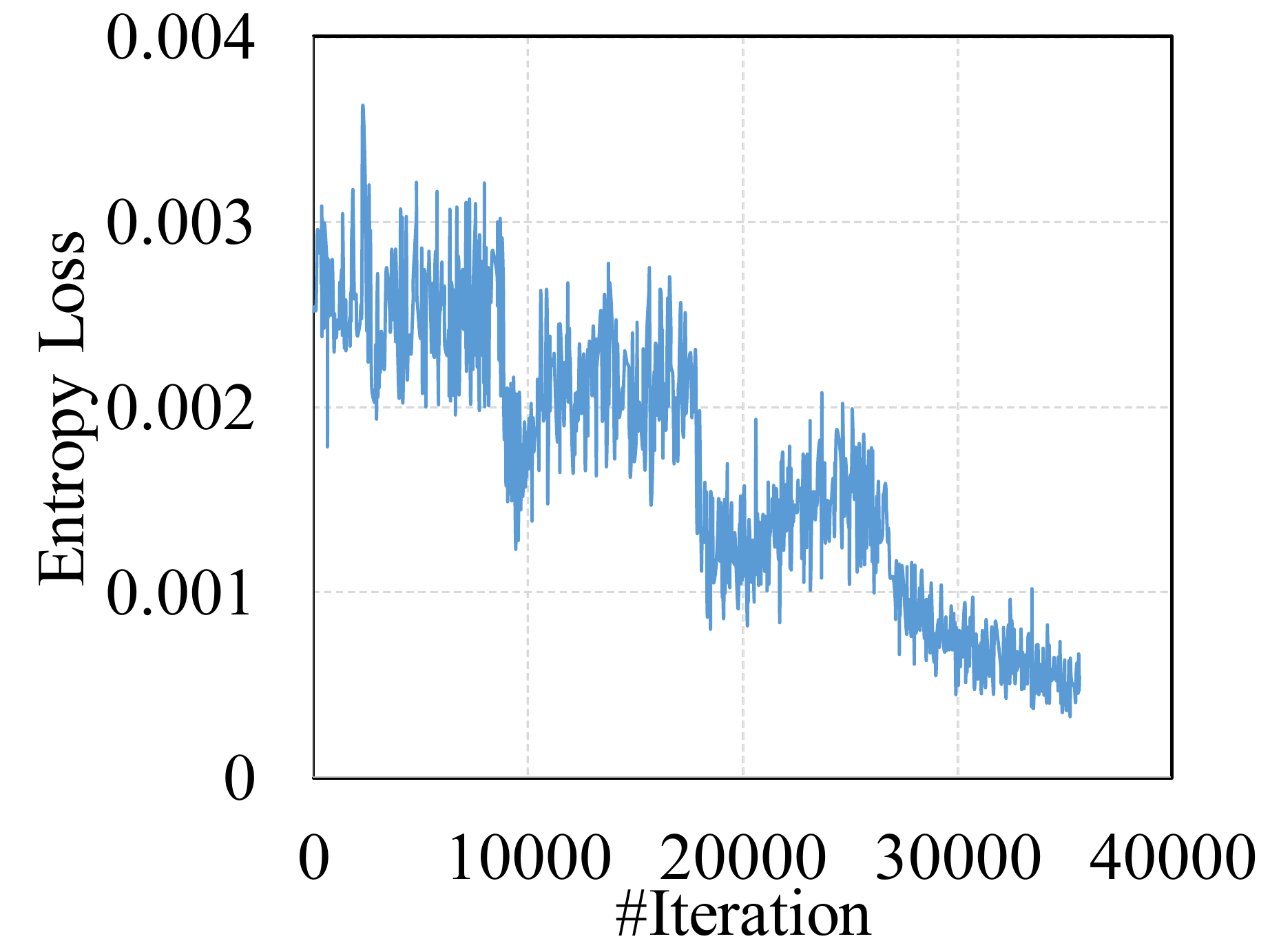}}
	\subfloat[]{\label{fig:ablation-d:sub3}\includegraphics[width=0.3\linewidth]{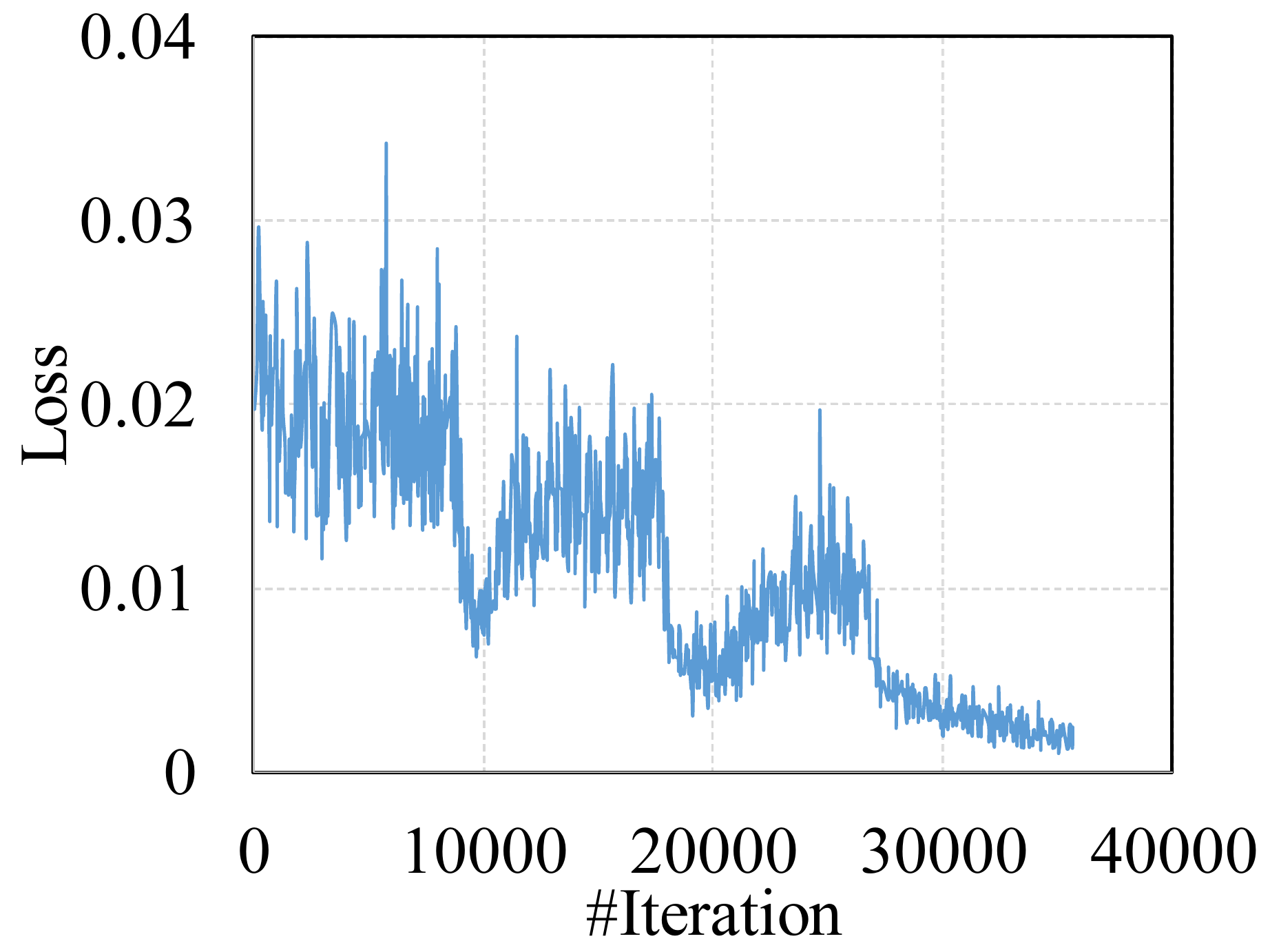}} 
	\caption{\protect\subref{fig:ablation-d:sub1}, \protect\subref{fig:ablation-d:sub2}, \protect\subref{fig:ablation-d:sub3} show how $\mathcal{L}_c$, $\mathcal{L}_e$, $\mathcal{L}$ change by training D2 with only reducing learning rate on CIFAR-10 (column d in Table~\ref{tab:ablation-result-table}), respectively. Reducing learning rate occurs at 8750, 17500, and 26250 iterations. It can make loss decrease. But the learning algorithm is still impacted by the equality condition bias. }
	\label{fig:ablation-d} 
\end{figure}

Now we validate our framework by an ablation study on CIFAR-10 with the Shake-Shake backbone and 4000 labeled images. All experiments used the same data splits and ran once. And they all used the first stage to initialize D2 and the third stage to finetune the network. Table~\ref{tab:ablation-result-table} presents the results and the error rates are produced by the last epoch of the third stage. Different columns denote using different strategies to train D2 in the second stage.
First, without R2 (column a), the error rate of a basic D2 learning is $6.71\%$, which is already competitive with state-of-the-arts. Next, we repeated the second stage without reprediction or reducing learning rate (column b). That means the network is trained by the first stage, the second stage, repeat the second stage, and the third stage. This network achieved a $6.37\%$ error rate, which demonstrates training D2 for more epochs can boost performance and the network will not overfit easily. Repeating the second stage with reprediction (column c) could make the error rate even lower, to $6.23\%$. But, without reducing the learning rate, $\mathcal{L}$ did not decrease (cf. Figure~\ref{fig:ablation-c:sub3}). On the other hand, repeating the second stage and reducing the learning rate (column d) can get better results ($5.94\%$). However, only reducing the learning rate cannot remove the impact of the equality constraint bias. At last, applying both strategies (column e) improved the results by a large margin to $5.78\%$.

Table~\ref{tab:ablation-alpha} presents the results with different $\alpha$. $\beta$ is $0.03$ in all experiments. We find setting $\alpha=0.2$ will achieve a better performance and a large $\alpha$ may degrade the performance. Table~\ref{tab:ablation-beta} shows the results with different $\beta$ when setting $\alpha=0.1$. Compared with $\alpha$, our method is robust to $\beta$. The highest error rate is $5.83\%$ and the lowest error rate is $5.62\%$. There is roughly $0.2\%$ between them. Table~\ref{tab:ablation-lambda} studies the sensitiveness of our method under different $\lambda$ which is the learning rate for updating pseudo-labels (cf. Equation~\ref{updating-formula}). Intuitively, the pseudo-labels can hardly be updated by a small $\lambda$. And with a large $\lambda$, the pseudo-labels will always be the same as the predictions and thus the training will fail. In practice, we find our method is not sensitive to $\lambda$. With $\lambda=1000$, the error rate is $5.85\%$, only slightly worse than $\lambda=4000$. When setting $\lambda=5000$, the performance is even better. Overall, R2-D2 is robust to these hyperparameters. And when apply R2-D2, we suggest that $\alpha=0.1, \beta=0.03$ and $\lambda=4000$ is a safe starting point to tune these hyperparameters. All experiments in the rest of our paper used $\alpha=0.1, \beta=0.03,\lambda=4000$. Please note that we did not carefully tune these hyperparameters. Error rates of R2-D2 may be lower than those reported in this paper if we tune them carefully.

Table~\ref{tab:ablation-loss} shows the results with different $\mathcal{L}_c$. Note that our loss function is defined as $\mathcal{L}=\alpha\mathcal{L}_c + \beta\mathcal{L}_e$. The loss function determines how the network parameters and pseudo-labels update. That means different $\mathcal{L}_c$ result in different updating formulas of pseudo-labels. The default $\mathcal{L}_c$ is $KL(\hat{\mathbf{p}}||\tilde{\mathbf{p}})$ and the updating formula is Equation~\ref{updating-formula}. When set $\mathcal{L}_c=KL(\tilde{\mathbf{p}}||\hat{\mathbf{p}})$, the gradients of $\mathcal{L}$ with respect to $\tilde{y}_n$ is 
\begin{equation}
\frac{\partial\mathcal{L}}{\partial\tilde{y}_n}
=\alpha
\tilde{p}_n
[\log \tilde{p}_n
- \log \hat{p}_n
- \sum_{k=1}^{N}\tilde{p}_k(\log\tilde{p}_k - \log\hat{p}_k)]
\,,
\end{equation}
where $\hat{\mathbf{p}} = \sigma(\hat{\mathbf{y}})$ and $\tilde{\mathbf{p}} = \sigma(\tilde{\mathbf{y}})$. With $\mathcal{L}_c=\|\tilde{\mathbf{p}} - \hat{\mathbf{p}}\|_2^2$, the gradients of $\mathcal{L}$ with respect to $\tilde{y}_n$ is 
\begin{equation}
\frac{\partial\mathcal{L}}{\partial\tilde{y}_n}
=2\alpha
\tilde{p}_n
[\tilde{p}_n
- \hat{p}_n
- \sum_{k=1}^{N}\tilde{p}_k(\tilde{p}_k - \hat{p}_k)]
\,.
\end{equation}
Note that due to the sigmoid transform, $\tilde{\mathbf{p}}$ and $\hat{\mathbf{p}}$ are much smaller than $\tilde{\mathbf{y}}$, so all of them need a large $\lambda$ to update pseudo-labels.
The experimental results demonstrate superior performance of R2-D2 with $\mathcal{L}_c = KL(\hat{\mathbf{p}}||\tilde{\mathbf{p}})$. It obtains $2.28\%$ lower error rate than $KL(\tilde{\mathbf{p}}||\hat{\mathbf{p}})$ and $0.57\%$ lower error rate than $\|\tilde{\mathbf{p}} - \hat{\mathbf{p}}\|_2^2$.

\subsection{Results on ImageNet}

Table~\ref{tab:ImageNet-result-table} shows our results on ImageNet with 10\% labeled samples. The setup followed that in \citet{DCT_2018_ECCV}. The image size in training and testing is $224\times 224$. For the fairness of comparisons, the error rate is from single model without ensembling. We use the result of the last epoch. Our experiment is repeated three times with different random subsets of labeled training samples. The Top-1 error rates are $41.64$, $41.35$, and $41.65$, respectively. The Top-5 error rates are $19.53$, $19.60$, and $19.44$, respectively. R2-D2 achieves significantly lower error rates than Stochastic Transformations \citep{Stochastic_Transformations} and VAE~\citep{VAE}, although they used the larger input size $256\times 256$. With the same backbone and input size, R2-D2 obtains roughly $5\%$ lower Top-1 error rate than that of DCT~\citep{DCT_2018_ECCV} and $7.5\%$ lower Top-1 error rate than that of Mean Teacher~\citep{Mean_teacher}. R2-D2 outperforms the previous state-of-the-arts by a large margin. The performances of Mean Teacher~\citep{Mean_teacher} with ResNet-18~\citep{ResNet} is quoted from \citet{DCT_2018_ECCV}.

Self-supervised learning is another way to utilize unlabeled data. In self-supervised learning, it needs to define a pretext task to train the network. By solving the pretext task, we expect the network can learn better representations. And with the better representations, the network finetuned by a few labeled data can get a better performance than training it from scratch. Recently, RotNet~\citep{RotNet} is a simple and promising self-supervised learning technique. RotNet uses recognizing the image rotation as the pretext task. We can combine R2-D2 with RotNet. First, we train the network by recognizing the image rotation ($0^\circ$, $90^\circ$, $180^\circ$, $270^\circ$) with all images (labeled images and unlabeled images). Then, we replace its FC layer by 1000-class weights of random initialization and use R2-D2 to train the network. Table~\ref{tab:ImageNet-result-table} shows the results and using RotNet pretrained weight can improve roughtly $1\%$ without bells and whistles.

\begin{table*}
	\caption{Error rates (\%) on the validation set of ImageNet benchmark with 10\% images labeled. ``-'' means that the original papers did not report the corresponding error rates. ResNet-50 is used. }
	\label{tab:ImageNet-resnet50-table}
	\centering
	\begin{tabular}{c|llcc}
		\toprule
		&
		Method	& Backbone		&	Top-1	&	Top-5	\\
		\midrule
		\multirow{2}*{Supervised} &
		100\% Supervised	&	ResNet-50	&	23.75	&  7.23	\\
		&
		10\% Supervised		&	ResNet-50	&	45.55	&  20.73	\\
		\midrule
		\multirow{6}*{Semi-supervised} &
		Pseudo-label		&	ResNet-50v2	&     -		&  17.59	\\
		
		&							
		VAT					&	ResNet-50v2	&	-		&  17.22	\\
		
		&							
		VAT + EntMin		&	ResNet-50v2	&	-		&  16.61	\\
		
		&							
		S$^4$L-Rotation		&	ResNet-50v2	&	-		&  16.18	\\
		
		&							
		S$^4$L-Exemplar		&	ResNet-50v2	&	-		&  16.28	\\	
		
		&
		R2-D2				&	ResNet-50	&\textbf{34.01}&\textbf{14.07}\\
		\bottomrule
	\end{tabular}
\end{table*}

Table~\ref{tab:ImageNet-resnet50-table} shows our results with the ResNet-50 backbone network. The setup is the same as that of ResNet-18. And our experiment was run for once. ResNet-50 denotes the regular type~\citep{ResNet} and ResNet-50v2 denotes the pre-activation variants~\citep{ResNetv2}. The results of Pseudo-label, VAT, VAT + EntMin, S$^4$L-Rotation, S$^4$L-Exemplar are quoted from \citet{S4L}. And R2-D2 is significantly better than them. Note that \citet{S4L} proposed MOAM (Mix Of All Models) and got a better performance. However, they used a $4\times$ wider model as backbone network and it is not fair to compare ours with MOAM's results.

\subsection{Results on CIFAR-100}

\begin{table*}
	\caption{Error rates (\%) on CIFAR-100 benchmark with 10000 images labeled.}
	\label{tab:CIFAR-100-result-table}
	\centering
	\begin{tabular}{c|lcr@{.}l}
		\toprule
		&Method							&	Backbone	&	\multicolumn{2}{c}{Error rates (\%)}		\\
		\midrule
		\multirow{2}*{Supervised} &
		
		100\% Supervised				&	ConvLarge		&	$26$&$42 \pm 0.17$ 			\\
		&
		Using 10000 labeled images only	&	ConvLarge		&	$38$&$36 \pm 0.27$	\\
		\midrule
		\multirow{6}*{Semi-supervised} &		
		Temporal Ensembling				&	ConvLarge		&	$38$&$65 \pm 0.51$	\\
		&
		LP								&	ConvLarge		&	$38$&$43 \pm 1.88$	\\
		&
		Mean Teacher					&	ConvLarge		&	$36$&$08 \pm 0.51$	\\
		&
		LP + Mean Teacher				&	ConvLarge		&	$35$&$92 \pm 0.47$	\\
		&
		DCT								&	ConvLarge		&	$34$&$63 \pm 0.14$	\\
		&
		R2-D2 							&	ConvLarge		&$\textbf{32}$&$\textbf{87}\pm\textbf{0.51}$\\
		\bottomrule
	\end{tabular}
\end{table*}

Table~\ref{tab:CIFAR-100-result-table} presents experimental results on CIFAR-100 with 10000 labeled samples. All methods used ConvLarge for fairness of comparisons and did not use ensembling. The error rate of R2-D2 is the average error rate of the last epoch over five random data splits. The results of 100\% Supervised is quoted from \citet{Temporal_Ensembling}. Using 10000 labeled images achieved $38.36\%$ error rates in our experiments. With unlabeled images, R2-D2 produced a $32.87\%$ error rate which is lower than others (e.g., Temporal Ensembling, LP~\citep{cvpr2019_pseudo_label}, Mean Teacher~\citep{Mean_teacher}, LP + Mean Teacher~\citep{cvpr2019_pseudo_label}, and DCT). The performances of Mean Teacher~\citep{Mean_teacher} is quoted from \citet{cvpr2019_pseudo_label}.

\subsection{Results on CIFAR-10}

\begin{table}
	\caption{Error rates (\%) on CIFAR-10 benchmark with 4000 images labeled.}
	\label{tab:CIFAR-10-result-table}
	\centering
	\small
	\begin{tabular}{lcr@{.}l}
		\toprule
		Method							&	Backbone	&	\multicolumn{2}{c}{Error rates (\%)}		\\
		\midrule
		
		100\% Supervised				&	Shake-Shake		&	$2$&$86$			\\
		
		Only 4000 labeled images 		&	Shake-Shake		&	$14$&$90\pm0.28$	\\
		
		\midrule
		
		Mean Teacher					&	ConvLarge		&	$12$&$31\pm0.28$	\\
		
		Temporal Ensembling				&	ConvLarge		&	$12$&$16\pm0.24$	\\
		
		VAT+EntMin						&	ConvLarge		&	$10$&$55\pm0.05$	\\
		
		DCT with 8 Views				&	ConvLarge		&	$8$&$35\pm0.06$	\\
		
		Mean Teacher					&	Shake-Shake		&	$6$&$28\pm0.15$	\\
		
		HybridNet						&	Shake-Shake		&	$6$&$09$			\\
		
		R2-D2 							&	Shake-Shake		&$\textbf{5}$&$\textbf{72}\pm\textbf{0.06}$	\\
		\bottomrule
	\end{tabular}
\end{table}

We evaluated the performance of R2-D2 on CIFAR-10 with 4000 labeled samples. Table~\ref{tab:CIFAR-10-result-table} presents the results. Following \citet{Mean_teacher,HybridNet}, we used the Shake-Shake network~\citep{shakeshake} as the backbone network. Overall, using Shake-Shake backbone network can achieves lower error rates than using ConvLarge. Our experiment was repeated five times with different random subsets of labeled training samples. We used the test error rates of the last epoch. After the first stage, the backbone network produced the error rates 14.90\%, which is our baseline using 4000 labeled samples. With the help of unlabeled images, R2-D2 obtains an error rate of $5.72\%$.Compared with Mean Teacher~\citep{Mean_teacher} and HybridNet~\citep{HybridNet}, R2-D2 achieves lower error rate and produces state-of-the-art results. 

\subsection{Results on SVHN}

\begin{table}
	\caption{Error rates (\%) on SVHN benchmark with 1000 images labeled.}
	\label{tab:SVHN-result-table}
	\centering
	\small
	\begin{tabular}{lcr@{.}l}
		\toprule
		Method						&	Backbone	&	\multicolumn{2}{c}{Error rates (\%)}		\\
		\midrule
		
		100\% Supervised			&	ConvLarge	&	$2$&$88\pm0.03$	\\
		
		Only 1000 labeled images 	&	ConvLarge	&	$11$&$27\pm0.85$	\\
		\midrule
		Temporal Ensembling			&	ConvLarge	&	$4$&$42\pm0.16$	\\
		
		VAdD (KL)         & ConvLarge & $4$&$16\pm0.08$ \\
		
		Mean Teacher				&	ConvLarge	&	$3$&$95\pm0.19$	\\	
		
		VAT+EntMin					&	ConvLarge	&	$3$&$86\pm0.11$	\\
		
		VAdD (KL) + VAT       & ConvLarge & $3$&$55\pm0.05$ \\
		
		DCT with 8 Views			&	ConvLarge	&$\textbf{3}$&$\textbf{29}\pm\textbf{0.03}$\\
		
		R2-D2						&	ConvLarge	&	$3$&$64\pm0.20$			\\
		\bottomrule
	\end{tabular}
\end{table}

We tested R2-D2 on SVHN with 1000 labeled samples. The results are shown in Table~\ref{tab:SVHN-result-table}.
Following previous works~\citep{Temporal_Ensembling,Mean_teacher,VAT,DCT_2018_ECCV}, we used the ConvLarge network as the backbone network. The result we report is average error rate of the last epoch over five random data splits. On this task, the gap between 100\% supervised and many SSL methods (e.g., VAT+EntMin~\citep{VAT}, VAdD (KL)+VAT~\citep{VAdD}, Deep Co-Training~\citep{DCT_2018_ECCV}, and R2-D2) is less than 1\%. Only Deep Co-Training with 8 Views~\citep{DCT_2018_ECCV} and VAdD (KL)+VAT slightly outperform R2-D2. Compared with other methods (e.g., Temporal Ensembling, Mean Teacher, and VAT, R2-D2 produces a lower error rate. Note that on the large-scale ImageNet, R2-D2 significantly outperformed Deep Co-Training. VAdD have not be evaluated on ImageNet in their paper.

\subsection{Combine R2-D2 with other SSL method}

\begin{table}
	\caption{Error rates (\%) on CIFAR-10 benchmark with 4000 images labeled. $\dag$ denotes results reported in the original papers. MT means Mean Teacher. }
	\label{tab:CIFAR-10-SWA-R2D2-table}
	\centering
	\small
	\begin{tabular}{lcc}
		\toprule
		Method							&	Backbone	&	Error rates (\%)	\\
		\midrule
		
		MT + fast-SWA (1200)$^\dag$		&	ConvLarge		&	$9.05$	\\
		
		MT + fast-SWA (1200)			&	ConvLarge		&	$9.70$	\\
		
		MT + fast-SWA (1200) + R2-D2 	&	ConvLarge		&	$9.27$	\\
		
		\midrule
		
		MT + SWA (1200)$^\dag$			&	ConvLarge		&	$9.38$	\\
		
		MT + SWA (1200)					&	ConvLarge		&	$9.37$	\\
		
		MT + SWA (1200) + R2-D2 		&	ConvLarge		&	$9.11$	\\
		
		\bottomrule
	\end{tabular}
\end{table}

Now, we study if R2-D2 can boost other SSL methods' performance. Note that the overall R2-D2 algorithm consists of three stages. In the raw first stage, we only use the labeled images to train the backbone network. Combining R2-D2 with other SSL method, we can use other SSL method as the first stage in our algorithm. That means we use aother SSL method to train the network with labeled and unlabeled images. Then, we use the trained network to initialize our framework and continue to train the network by R2-D2. Table~\ref{tab:CIFAR-10-SWA-R2D2-table} presents the results and R2-D2 indeed boosts other SSL method performance. Our implementation of MT~+~fast-SWA~(1200)~\citep{SWA} achieve $9.70\%$ error rate. And with the help of R2-D2, the error rate is $9.27\%$ which is $0.43\%$ lower. Combining R2-D2 with MT~+~SWA~(1200)~\citep{SWA} results in $9.11\%$ error rate which is better than $9.37\%$ of MT~+~SWA~(1200).

\subsection{Realistic evaluation of R2-D2}

\begin{table}
	\caption{Error rates (\%) on CIFAR-10 benchmark with 4000 labeled images and balanced/unbalanced unlabeled images. All expriments use the Shake-Shake backbone network.}
	\label{tab:CIFAR-10-unbalance-table}
	\centering
	\small
	\begin{tabular}{ccc}
		\toprule
		\multirow{2}*{Unlabeled data}	&	\multicolumn{2}{c}{Error rates (\%)}	\\
		& Mean Teacher & R2-D2 \\
		\midrule
		
		46000 balanced 		&	$6.60$	&	$5.72$	\\
		
		23000 balanced 		&	$9.28$	&	$8.08$	\\
		
		23000 unbalanced 	&	$9.72$	&	$9.52$	\\
		
		\bottomrule
	\end{tabular}
\end{table}

\begin{table}
	\caption{Error rates (\%) on CIFAR-10 benchmark with 4000 labeled images and open world assumption. All expriments use the Shake-Shake backbone network. }
	\label{tab:CIFAR-10-open-table}
	\centering
	\small
	\begin{tabular}{ccr@{.}lr@{.}l}
		\toprule
		\multicolumn{2}{c}{Unlabeled data} & \multicolumn{4}{c}{Error rates (\%)} \\
		CIFAR-10 & CIFAR-100 & \multicolumn{2}{c}{Mean Teacher} & \multicolumn{2}{c}{R2-D2}	\\
		\midrule
		
		46000 balanced		&	0 	&	$6$&$60$	&	$5$&$72$	\\
		
		46000 balanced		& 5000 	&	$7$&$00$	&	$6$&$41$	\\
		
		\midrule
		
		23000 balanced 		& 0		&	$9$&$28$	&	$8$&$08$	\\
		
		23000 balanced		& 12000	&	$9$&$68$	&	$8$&$99$	\\
		
		23000 unbalanced 	& 0		&	$9$&$72$	&	$9$&$52$	\\
		
		23000 unbalanced	& 12000	&	$10$&$85$	&	$10$&$41$	\\
		
		\midrule
		
		23000 balanced		& 23000	&	$28$&$66$	&	$15$&$32$	\\
		
		23000 unbalanced	& 23000	&	$28$&$15$	&	$18$&$05$	\\
		
		\bottomrule
	\end{tabular}
\end{table}

In this section, we evaluate R2-D2 under more realistic experiment setting. As \citet{Realistic_Eval} pointed out, in ``real-world'', the unlabeled data may be unbalanced and even contain a different
distribution of classes than the labeled data. First, we study the sensitiveness of our method when trained with unbalanced unlabeled data. Table~\ref{tab:CIFAR-10-unbalance-table} shows the results. ``46000 balanced'' means the typical setting that is using 4000 labeled data and 46000 unlabeled data of CIFAR-10. ``23000 balanced'' denotes using 23000 balanced unlabel data (2300 per class). At last, we produce 23000 unbalanced unlabeled data by random sampling. Each class contains 2770, 3452, 2042, 4062, 4047, 758, 590, 2588, 2201, 490 images, respectively. According to the experimental results, R2-D2 and Mean Teacher are more sensitive to the number of unlabeled data. When using a half but balanced unlabeled data, the performances are degraded by $2.36\%$ and $2.68\%$, respectively. However, the gaps of error rates between ``23000 balanced'' and ``23000 unbalanced'' are only $1.44\%$ and $0.44\%$, respectively.

Finally, we devise the experiments to simulate the situation that the unlabeled data contain a different distribution of classes than the labeled data. And we call it ``open world assumption''. Because the classes of CIFAR-100 are different from that of CIFAR-10, we select some images in CIFAR-100 to add to the unlabeled data. Table~\ref{tab:CIFAR-10-open-table} presents the results. The model performance can often be significantly degraded when adding CIFAR-100 images. Because we predict the pseudo-labels of unlabeled data repetitively, we can use the entropy of pseudo-labels to estimate if the unlabeled images belong to CIFAR-10. An effective remedy is to throw away $10\%$ unlabeled data whose pseudo-label entropy are larger than others after each reprediction. With this remedy, R2-D2 achieves better performance than Mean Teacher. However, when adding 23000 CIFAR-100 images, the error rates of both methods are higher than that of only using labeled data. It is still an open problem to make sure the network indeed benefit from the unlabeled data whose distribution is different from the labeled data.

\section{Conclusion}
In this paper, we proposed R2-D2, a method for semi-supervised deep learning. D2 uses label probability distributions as pseudo-labels for unlabeled images and optimizes them during training. Unlike previous SSL methods, D2 is an end-to-end framework, which is independent of the backbone network and can be trained by back-propagation. Based on D2, we give a theoretical support for using network predictions as pseudo-labels. However, pseudo-labels will become flat during training. We analyzed this problem both theoretically and experimentally, and proposed the R2 remedy for it. At last, we tested R2-D2 on different datasets. The experiments demonstrated superior performance of our proposed methods. On large-scale dataset ImageNet, R2-D2 achieved about $5\%$ lower error rates than that of previous state-of-the-art. In the future, we will further explore the combination of unsupervised feature learning and semi-supervised learning, and deep SSL in the open world assumption.

%\begin{acknowledgements}
%\end{acknowledgements}


\bibliographystyle{spbasic}      % basic style, author-year citations
\bibliography{ref}   % name your BibTeX data base

%\begin{thebibliography}{}
% etc
%\end{thebibliography}

\end{document}